\documentclass{article}
\usepackage{microtype}
\usepackage[dvipsnames,table]{xcolor}
\usepackage{graphicx}
\usepackage[bibliography=common,appendix=append]{apxproof}

\usepackage{subfigure}
\usepackage{booktabs}

\usepackage{hyperref}
\usepackage[accepted]{icml2024}

\usepackage{amsmath}
\usepackage{amssymb}
\usepackage{mathtools}
\usepackage{amsthm}
\usepackage[normalem]{ulem}

\usepackage{emile}
\usepackage{enumitem}
\setlist[enumerate]{noitemsep, topsep=0pt}

\usepackage[capitalize,noabbrev]{cleveref}

\newcommand{\Wdim}{{n}}

\newcommand{\lengthcube}{2}

\newcommand{\implicant}{{\bw_D}}
\newcommand{\altimplicant}{{\bw_E}}

\newcommand{\indep}{{\perp \!\!\! \perp}}
\newcommand{\independent}{{\distributions^\indep}}
\newcommand{\Pimplicant}{{C_{\implicant}}}
\newcommand{\cell}{{\stackrel{\circ}{C}}}
\newcommand{\cellint}{{\stackrel{\circ}{I}}}
\newcommand{\cellimplicant}{\cell_{\implicant}}
\newcommand{\Paltimplicant}{{C_{\altimplicant}}}
\newcommand{\cubepossible}{{C_{\knowledge{\by}}}}
\newcommand{\possindependent}{{\distributions^ \indep_{\knowledge{\by}}}}
\newcommand{\possmix}[1]{{\mathcal{P}_{k, \knowledge{\by}}}}
\newcommand{\distributions}{{\Delta}}

\newcommand{\poss}[1]{{\Delta_{\knowledge{#1}}}}

\newcommand{\knowledgevar}{\knowledge{\by}=1}
\newcommand{\varprob}{{\mu}}
\newcommand{\probs}{{\boldsymbol{\mu}}}
\newcommand{\prob}{{\mu}}

\theoremstyle{plain}
\newtheoremrep{theorem}{Theorem}[section]
\newtheorem{proposition}[theorem]{Proposition}
\newtheorem{lemma}[theorem]{Lemma}

\theoremstyle{definition}
\newtheorem{definition}[theorem]{Definition}

\theoremstyle{remark}

\theoremstyle{definition}
\newtheorem{example}[theorem]{Example}

\usepackage{todonotes}
\makeatletter
\newcommand*\iftodonotes{\if@todonotes@disabled\expandafter\@secondoftwo\else\expandafter\@firstoftwo\fi}
\makeatother

\definecolor{edolime}{rgb}{0.9,1,0.3}

\title{On the Independence Assumption in Neurosymbolic Learning}

\begin{document}

\twocolumn[
\icmltitle{On the Independence Assumption in Neurosymbolic Learning}

\begin{icmlauthorlist}
	\icmlauthor{Emile van Krieken}{yyy}
	\icmlauthor{Pasquale Minervini}{yyy}
        \icmlauthor{Edoardo M. Ponti}{yyy}
        \icmlauthor{Antonio Vergari}{yyy}
\end{icmlauthorlist}
	
\icmlaffiliation{yyy}{School of Informatics, University of Edinburgh}

\icmlcorrespondingauthor{Emile van Krieken}{Emile.van.Krieken@ed.ac.uk}

\icmlkeywords{Neurosymbolic AI, Probabilistic Reasoning, Topology, optimisation}

\vskip 0.3in
]
\printAffiliationsAndNotice{} 

\begin{abstract}
	State-of-the-art neurosymbolic learning systems use probabilistic reasoning to guide neural networks towards predictions that conform to logical constraints over symbols. 
	Many such systems assume that the probabilities of the considered symbols are conditionally independent given the input to simplify learning and reasoning. 
	We study and criticise this assumption, highlighting how it can hinder optimisation and prevent uncertainty quantification. 
        We prove that loss functions bias conditionally independent neural networks to become overconfident in their predictions. As a result, they are unable to represent uncertainty over multiple valid options.  
	Furthermore, we prove that these loss functions are difficult to optimise: they are non-convex, and their minima are usually highly disconnected.
        Our theoretical analysis gives the foundation for replacing the conditional independence assumption and designing more expressive neurosymbolic probabilistic models.

\end{abstract}

\section{Introduction}
\emph{Neurosymbolic learning} studies neurosymbolic models that combine neural perception and symbolic reasoning \citep{manhaeveNeuralProbabilisticLogic2021,xuSemanticLossFunction2018,badreddineLogicTensorNetworks2022}. These models use logical constraints and data to create a loss function for learning neural perception models \citep{giunchigliaDeepLearningLogical2022}. When used effectively, neurosymbolic learning methods can use these constraints to %
improve data efficiency.
However, researchers in the neurosymbolic learning community have found optimising the parameters of the perception models challenging~\citep{marconatoNeuroSymbolicReasoningShortcuts2023,vankriekenAnalyzingDifferentiableFuzzy2022,manhaeveNeuralProbabilisticLogic2021}. A major underlying reason is that neurosymbolic learning cannot provide exact feedback on how the neural perception model should behave. We highlight this issue with a simple example.

\begin{figure}[t]
	\centering
	\includegraphics[width=\linewidth]{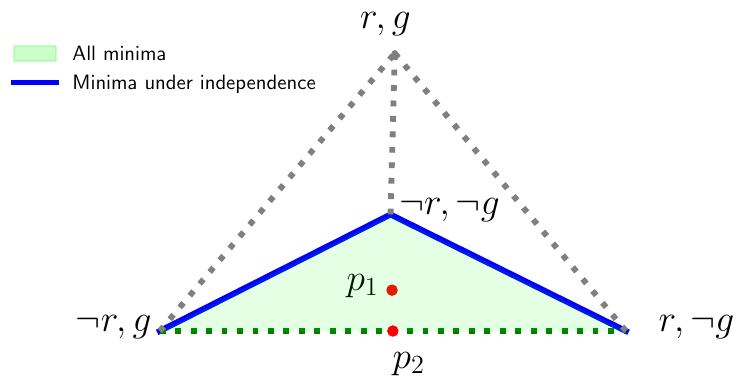}
	\caption{\textbf{The conditional independence assumption discards valid and potentially meaningful solutions.} 
	The tetrahedron (a 3-dimensional probability simplex) represents the distributions over the options of the problem in Example~\ref{ex:traffic-lights}: $r$ refers to the red light and $g$ to the green light. The green triangle represents distributions that assign zero probability to $r\wedge g$. The blue lines are the distributions in the green triangle that an independent distribution can represent. The left (resp. right) blue line represents the distributions where the probability of $r$ (resp. $g$) is zero. Independent distributions cannot represent distributions in the dotted green line, such as $p_2$ that assigns equal probability to only the green or only the red light being on.%
    \scalebox{0.01}{minima immoralia}
 } 
	\label{fig:intro-fig}
\end{figure}

\begin{example}
	\label{ex:traffic-lights}
	We consider a perception model responsible for recognising the red and green lights on a traffic light. It sees a traffic light that it believes to be simultaneously red and green. A constraint specifies this is impossible, and the neurosymbolic loss should penalise this. There are three possible worlds: the model can output that the red light is on, the green light is on, or neither. How do we choose among these?
\end{example}
The set of all beliefs can be represented by the tetrahedron in Figure \ref{fig:intro-fig}. We argue that the perception model should be able to express \emph{uncertainty} over these three worlds, as there is no evidence to conclude which one is correct. This corresponds to the distributions in the green triangle at the bottom of the figure. We should leave determining any further preference to the provided data as the constraint only specifies what worlds are \emph{possible} but does not specify which one is \emph{correct}. 

The majority of probabilistic methods for neurosymbolic learning rely on a strong assumption: namely, that the different symbols of the world are \emph{independent} when conditioned on input data \citep{manhaeveDeepProbLogNeuralProbabilistic2018,xuSemanticLossFunction2018}. This means that for some input image, a conditionally independent perception model predicts two probabilities: one for the green light being on and one for the red light being on. 

What do we lose when we take this conditional independence assumption? There is recent experimental evidence that suggests that using \emph{expressive} perception models over conditionally independent ones improves performance on neurosymbolic tasks \citep{ahmedSemanticProbabilisticLayers2022,ahmedPseudosemanticLossAutoregressive2023,pryorNeuPSLNeuralProbabilistic2023}. We theoretically justify these results: the conditional independence assumption causes neurosymbolic methods to be biased towards deterministic solutions. 
This is because minima of neurosymbolic losses have to deterministically assign values to some variables. For instance, in Figure \ref{fig:intro-fig}, the blue lines represent distributions that state that either the red light is off or the green light is off.

With the goal of better understanding the impact of the conditional independence assumption, we provide a computable characterisation of what can be represented. %
We find that we can characterise this problem faithfully using tools from logic \citep{quineCoresPrimeImplicants1959} and computational homology \citep{kaczynskiComputationalHomology2004}, and prove that this bias towards determinism holds generally. Furthermore, our characterisation shows that the conditional independence assumption can lead to training objectives %
that are challenging to optimise due to heavily disconnected minima. %
Our analysis provides theoretical justifications for the benefits of using more expressive perception models: the ability to properly express uncertainty and smooth, convex loss landscapes.

\section{Background and Notation}
\label{sec:background}
\textbf{Probabilistic Neurosymbolic Learning.} 
We consider a probabilistic neurosymbolic learning (PNL) setting, where a probabilistic neural perception model $p_{\btheta}(\bw|\bx)$ with parameters $\btheta$ defines a distribution over \emph{worlds} $\bw \in \{0, 1\}^n$ (often called \emph{concepts} \citep{barbieroInterpretableNeuralSymbolicConcept2023,marconatoNeuroSymbolicReasoningShortcuts2023}) given high-dimensional inputs $\bx\in \mathcal{X}$.
A constraint $\knowledge{\by}: \{0, 1\}^n\rightarrow \{0, 1\}$ %
is a boolean function on worlds $\bw$. We say a world is \emph{possible} if $\knowledge{\by}(\bw)=1$ and assume $\knowledge{\by}$ has at least one possible world. 
The next example illustrates this setting.
\begin{example}[Learning with algorithms]
	\label{ex:mnist-add}
	MNIST Addition is a popular benchmark task in neurosymbolic learning \citep{manhaeveNeuralProbabilisticLogic2021}. $\mathcal{X}$ is the set of pairs of MNIST images. We represent worlds $\bw$ with $\Wdim=20$ variables $\{w_{1, 0}, ..., w_{1, 9}, w_{2, 0}, ..., w_{2, 9}\}$, where $w_{i,j}$ denotes the $i$th digit taking the value $j$.
	We have a set of labels representing possible sums $\mathcal{Y} =\{0, \ldots, 18\}$. The constraints $\knowledge{\by}_y$ enforce that exactly one of $w_{1, j}$ and one of $w_{2, k}$ is true, and ensures the pair of digits sums to the correct output: $\knowledge{\by}_{y}(\bw)=\exists_{j, k \in \{0, ..., 9\}}(j+k = y) \wedge w_{1, j} \wedge w_{2, k}$. Here, the constraints $\knowledge{\by}_y$ are parameterised by an observed output $y\in\mathcal{Y}$, which changes between inputs $\bx$.	
\end{example}

We compute the probability that the model $p_\btheta(\bw|\bx)$ satisfies the constraint $\knowledge{}$\footnote{With abuse of notation, we use the symbol $\knowledge{}$ both for the boolean function encoding the knowledge and for a binary random variable of the knowledge being true or not.} for input $\bx$ with:

\begin{equation}
	\label{eq:wmc}
	p_{\btheta}(\knowledgevar| \bx)=\sum_{\bw\in \{0, 1\}^n} p_{\btheta}(\bw|\bx) \knowledge{\by}(\bw).%
\end{equation}

Equation \ref{eq:wmc} is known as the (conditional) \emph{weighted model count (WMC)} in probabilistic and logical reasoning \citep{chaviraProbabilisticInferenceWeighted2008}. %
The $p_\btheta(\bw|\bx)$ term can be understood as a data-dependent factor that assigns probabilities to different worlds, while the $\knowledge{\by}(\bw)$ term is a constraint-dependent factor that filters out impossible worlds. %
Most loss functions based on WMC \citep{xuSemanticLossFunction2018,manhaeveNeuralProbabilisticLogic2021} minimise the negative logarithm of the WMC $\mathcal{L}(\btheta; \bx)=-\log p_\btheta(\knowledgevar| \bx)$, often called the \emph{semantic loss}.
See \cref{sec:related-work} for a discussion on these methods.

The majority of current PNL approaches assume that the probabilities $p_{\btheta}(w_i=1|\bx)$ of variables $w_i$ being true are independent when conditioned on $\bx$. Then, perception models only have to predict $n$ parameters in $[0, 1]^n$ instead of a parameter for each of the $2^n$ worlds, i.e., 
\begin{equation}
	\label{eq:indep-assumption}p_{\btheta}(\bw|\bx):=\prod_{i=1}^\Wdim p_\btheta(w_i|\bx).
\end{equation}
We call this the \emph{(conditional) independence assumption}.\footnote{In the rest of the paper, we refer to this \emph{conditional} independence assumption as just \say{the independence assumption} for readability.} 
PNL systems take advantage of this assumption 
to speed up inference, reduce the number of trainable parameters, and ease implementation \citep{xuSemanticLossFunction2018,manhaeveNeuralProbabilisticLogic2021,vankriekenANeSIScalableApproximate2023,ahmedSemanticProbabilisticLayers2022}.

\begin{example}[Semi-supervised learning with constraints]
	A common application of neurosymbolic learning is semi-supervised learning of the perception model $p_\btheta$. Here, we have a labelled dataset $\mathcal{D}_l=\{(\bx_i, \bw_i)\}_{i=1}^{|\mathcal{D}_l|}$ and an unlabelled dataset $\mathcal{D}_u=\{\bx_i\}_{i=1}^{|\mathcal{D}_u|}$. Here, $\knowledge{}$ is a conjunction of a set of constraints $\phi_i$ that relate the symbols in the world $\bw$. The semantic loss $\mathcal{L}(\btheta)$ over the unlabelled data $\mathcal{D}_u$ is often added as a regularisation term to a supervised loss function to bias the neural network towards solutions that predict possible worlds \citep{xuSemanticLossFunction2018}. %
\end{example}

\section{The issues with the independence assumption}
\label{sec:independence}
The main problem in our setting is how to learn the perception model $p_\btheta(\bw|\bx)$. The underlying assumption in neurosymbolic learning is that there is a true distribution over worlds $p^*(\bw|\bx)$. 
This distribution is unknown, and we cannot directly sample from $p^*(\bw|\bx)$ to train $p_\btheta(\bw|\bx)$. Instead, the feedback neurosymbolic learning methods provide is through the constraint $\knowledge{\by}$, which induces a \emph{set of possible worlds} 
$\mathcal{W}_{\knowledge{\by}}=\{\bw\in \{0, 1\}^n \mid \knowledge{\by}(\bw) = 1\}$. 
If the constraint is correct, then all worlds $\bw$ with non-zero probability $p^*(\bw|\bx)$ are in $\mathcal{W}_{\knowledge{\by}}$. 
Therefore, neurosymbolic learning methods should use the constraint $\knowledge{\by}$ as a \emph{filter} on what worlds are possible.
We aim to answer the following questions:
can particular parameterisations of $p_\btheta(\bw|\bx)$ implicitly bias the selection of possible worlds instead of just filtering out impossible ones? And if so, how does this hinder their ability to recover $p^*(\bw|\bx)$ via learning?

\subsection{The independence assumption biases towards deterministic solutions}
\label{sec:indepence}

\begin{figure}
	\centering
	\includegraphics[width=0.6\linewidth]{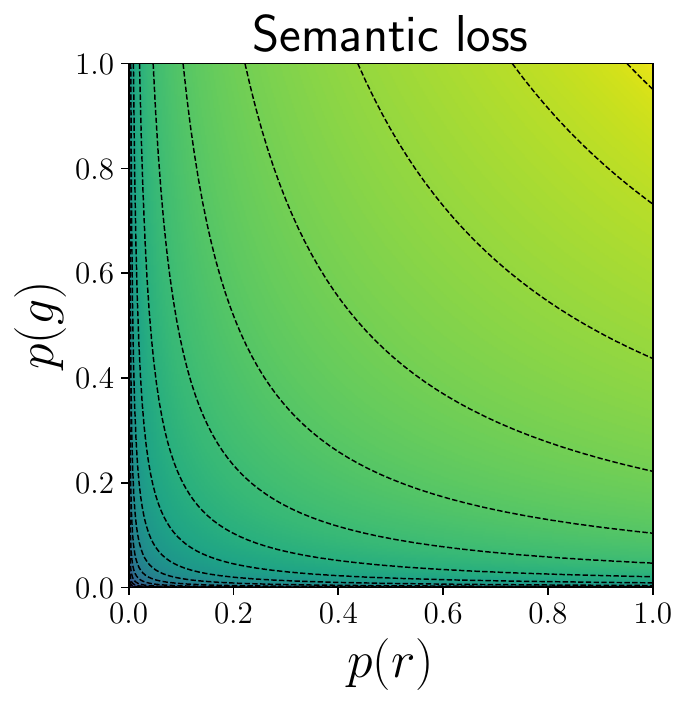}
	\caption{The loss landscape of the semantic loss for the traffic light problem -- brighter (resp. darker) regions correspond to higher (resp. lower) semantic loss values.} 
	\label{fig:semantic-loss-landscape}
\end{figure}

First, we show that common neurosymbolic learning methods are biased towards deterministic solutions. Returning to Example~\ref{ex:traffic-lights}, consider a simple setup with $\bw$ consisting of two binary variables $r$ and $g$ representing a red and green light, and a constraint $\knowledge{}=\neg r \vee \neg g$ which asserts that the red and green lights cannot be on simultaneously. 

In the remainder of the paper, we will fix the input $\bx$ and keep it implicit in our notation unless necessary. Using our formula $\knowledge{}$ in Equation \ref{eq:wmc}, we then get:\footnote{With some abuse of notation, we consider $r$ and $\neg g$ as events, that is, $p(r, \neg g):= p(r=1, g=0)$.}

\begin{equation} \label{eq:traffic-lights}
\begin{aligned}
	p_{\btheta}(\knowledgevar)&=p_{\btheta}(\neg r, \neg g) + p_{\btheta}(\neg r, g) + p_{\btheta}(r, \neg g) \\
	&= 1 - p_{\btheta}(r, g).
\end{aligned}
\end{equation}

We can maximise this probability by simply enforcing $p_\btheta(r, g)=0$. Then, the distribution over the remaining worlds can be arbitrary -- such distributions are represented by the green triangle in Figure \ref{fig:intro-fig}. However, taking the independence assumption over variables, we get:

\begin{equation*}
	p_{\btheta}(\knowledgevar)= 1 - p_\btheta(r) \cdot p_\btheta(g). 
\end{equation*}

We plot the semantic loss for an independent distribution in Figure \ref{fig:semantic-loss-landscape} as a function of $p_\btheta(r)$ and $p_\btheta(g)$. The semantic loss has its minima at the lines $p_\btheta(r)=0$ and $p_\btheta(g)=0$, biasing the model towards deterministically choosing either the red or green light being off, even though there is no evidence available to conclude this. Therefore, when optimising this function, we will come to a deterministic conclusion, which is wrong in a fraction of cases that depends on the real-world distribution of red and green lights being on.

Furthermore, an independent distribution cannot represent the beliefs $p_1$ and $p_2$ highlighted in Figure \ref{fig:intro-fig}, where $p_1$ is the uniform belief over the three possible worlds, while $p_2$ is the equal belief in only the red light being on or only the green light being on. In fact, we cannot represent any distribution that assigns a non-zero probability to either just the red or the green light being on: if the semantic loss is minimised, then an independent distribution cannot represent uncertainty among multiple equally valid options.

Does this bias towards determinism happen for all formulas $\varphi$? We prove that this is indeed the case using the concept of \emph{implicants} \citep{quineCoresPrimeImplicants1959}: an implicant assigns values to a subset of the variables $\{w_i\}_{i=0}^n$ such that it ensures the constraint $\knowledge{\by}$ is true. In our example, $\neg r$ is an implicant of $\knowledge{\by}$, since both $\neg r \wedge g$ and $\neg r \wedge \neg g$ are possible worlds. Our first theorem, which is formalised and proven in Section \ref{sec:possible-distributions}, generalises this result to all formulas $\knowledge{\by}$:

\begin{theorem}[Implicants determine minima, informal]
	\label{thm:implicants_informal}
	An independent distribution $p_\btheta(\bw)$ minimises the semantic loss if and only if it is deterministic for some variables, and those variables form an implicant of $\knowledge{\by}$.
\end{theorem}

We can, therefore, use the logical concept of implicants to study to what optima independent distributions converge. If the implicants are very restrictive, this greatly decreases the number of minima. The more restrictive the implicants of the formula are, the more the independent distributions will be biased towards deterministic solutions, and the less they will be able to quantify uncertainty.

\subsection{Minima under independence assumption are non-convex and disconnected}
Using the connection to implicants from Theorem \ref{thm:implicants_informal}, we develop a geometric characterisation of the independent distributions that minimise the semantic loss.
We emphasise that we study convexity and connectedness in the space of probability vectors over worlds, and not in the space of the parameters of the model $\btheta$. Our characterisation allows for an in-depth study of its topology using the tools of computational homology \citep{kaczynskiComputationalHomology2004}. This allows us to give the exact conditions on the constraints $\knowledge{\by}$ for which the minima are convex and connected. These conditions are valid only when we severely limit the types of constraints  we can use. Therefore, the resulting semantic loss functions are usually highly non-convex and disconnected, and so difficult to optimise. In contrast, for expressive distributions, the semantic loss is always convex, as we show in Section \ref{sec:geometric}.

\begin{theorem}[Convexity, informal]
	\label{thm:convexity_informal}
	The semantic loss is convex over the set of independent distributions only when the constraint $\knowledge{\by}$ is a formula of the form $\bigwedge_{i=1}^L l_i$, where each $l_i$ is a literal (a variable or its negation).
\end{theorem}
We discuss this in detail in Section \ref{sec:convexity}. The intuition behind this result is that such formulas provide direct supervision on $L$ variables, and give no supervision on the remaining $n-L$ variables. Therefore, the loss function is convex over the $L$ variables, and the remaining $n-L$ variables can be chosen arbitrarily. Note that this is a very restrictive condition: In many neurosymbolic settings, we have no direct supervision, and the constraint just acts as a filter on what worlds are possible.

\begin{theorem}[Connectedness, informal]
	\label{thm:connectedness_informal}
	The independent distributions that minimise the semantic loss are connected only if the implicants of the constraint $\knowledge{\by}$ form a connected graph between worlds.
\end{theorem}
In Section \ref{sec:connectedness}, we define a graph where the vertices are the possible worlds $\mathcal{W}_{\knowledge{\by}}$ that are connected when there is an implicant that \say{covers} both worlds. If this graph is connected, then so are the minima. This result is quite abstract, so we provide two examples to illustrate it. For the traffic light example, the graph is connected: The three possible worlds are connected through the implicants $\neg r$ and $\neg g$. However, for the MNIST Addition task (Example \ref{ex:mnist-add}), the graph contains no edges at all, which implies that the minima are a set of disconnected vertices.

\section{Characterising minima of the semantic loss}
In this section, we will develop the mathematical machinery to be able to characterise what it means for a distribution to be a minimum of the semantic loss, and, in particular, for independent distributions. 

Section \ref{sec:expressiveness} discusses the expressivity of distributions and introduces our notation. Section \ref{sec:possible-distributions} characterises the minima of the semantic loss for independent distributions. Section \ref{sec:representation} studies a minimal representation of those minima. Finally, Section \ref{sec:convexity} shows when this set is convex, and Section \ref{sec:connectedness} when this set is connected. Both turn out to be very rare.
We provide proof sketches for most theorems in the main text and leave the full proofs to Appendix \ref{appendix:proofs}. For ease of reading, we provide a table of notation in Table~\ref{tab:notation}.  

\subsection{Expressive distributions} 
\label{sec:expressiveness}
The \emph{expressiveness} of a perception model $p_\btheta(\bw|\bx)$ \citep{ahmedSemanticProbabilisticLayers2022} intuitively refers to how many distributions over worlds it can represent. A fully expressive perception model can represent any distribution $p(\bw|\bx)$. %

The set of all joint distributions over worlds is the $(2^\Wdim-1)$-(probability) simplex having the standard unit vectors $\mathbf{e}_i$ as vertices: 
$\distributions=\{\sum_{i=1}^{2^n}\alpha_i \mathbf{e}_i: \sum_{i=1}^{2^n} \alpha_i=1, \boldsymbol{\alpha} \geq \boldsymbol{0}\}\subset \mathbb{R}^ {2^ n}$.
The vertices $\mathbf{e}_i$ are one-hot representations of the worlds $\bw_i\in \{0, 1\}^\Wdim$. We fix an arbitrary ordering $\bw_1, ..., \bw_{2^n}$ throughout.
All probability distributions $p$ considered in this paper then correspond to the vector $(p(\bw_1), \ldots, p(\bw_{2^n}))$ in $\distributions$. 
With abuse of notation, we say $p\in \distributions$, referring to $p$ as both this vector and a distribution. 

We define \emph{possible distributions} $p\in \distributions$ as a distribution that assigns all probability mass to possible worlds. An equivalent statement is that $p(\bw)=0$ for all impossible worlds $\bw\in \{0, 1\}^n\setminus \mathcal{W}_{\knowledge{\by}}$  \citep{marconatoNotAllNeuroSymbolic2023}.
The \emph{set of all possible distributions} $\poss{\by} \subseteq \distributions$ is a $(|\mathcal{W}_{\knowledge{\by}}|-1)$-simplex formed from the standard unit vectors associated with the possible worlds $\mathcal{W}_{\knowledge{\by}}$. Since $\poss{\by}$ is a simplex, it is a convex set. Furthermore, the semantic loss $\mathcal{L}(p)$ is convex over the set of distributions $\distributions$ since the WMC is linear (see Appendix \ref{appendix:convexity} for a proof).

An \emph{expressive parameterisation} $\btheta \mapsto p_\btheta$ of the joint distribution can represent any distribution $p  \in \distributions$: For each input $\bx\in \mathcal{X}$, there is a parameter $\btheta$ such that $p(\bw)=p_{\btheta}(\bw|\bx)$. 
A (parameter-inefficient) fully expressive parameterisation is to predict a vector of $2^n$ logits, which is then mapped to $\distributions$ via softmax. %
Expressive parameterisations behave quite differently in the example discussed in Section \ref{sec:indepence}. They can minimise the probability of the constraint $\knowledge{\by}$ in Equation \ref{eq:traffic-lights} by simply setting $p_\btheta(r, g)=0$, and model any preference over the remaining three worlds. This prevents the model from having to deterministically choose that either the red or green light is off and allows it to represent uncertainty.

\subsection{When do independent parameterisations satisfy the constraint?}
\label{sec:possible-distributions}
We next study the properties of independent distributions that satisfy the constraint. 
An independent distribution $p_\probs$ is characterised by parameters $\probs$\footnote{If the perception model $p_{\btheta}$ is a neural network, $\probs$ would be the output of its last layer.
We drop the reference to $\btheta$ as it is not relevant for our analysis.
} 
in the $n$-hypercube $[0, 1]^ n$, where $\varprob_i$ is the probability that $w_i$ is true. To be precise, $p_\probs(w_i)=\mu_i^{w_i}\cdot(1-\mu_i)^{1-w_i}$ is the probability mass function of the Bernoulli random variable $w_i$. We now formally define \emph{implicants} \citep{quineCoresPrimeImplicants1959}, which are related to the deterministic components of $\probs$: 
\begin{definition}[Deterministic assignments]
	\label{def:detassign}
	A probability $\varprob_i\in [0, 1]$ is \emph{deterministic} if $\varprob_i\in \{0, 1\}$. Otherwise, $\varprob_i$ is \emph{stochastic}, that is, $\varprob_i\in (0, 1)$.
	A \emph{partial assignment} $\implicant=\{w_i\}_{i\in D}$ assigns values $\{0, 1\}$ to a subset of the variables $\bw$ indexed by $D\subseteq \{1, \ldots, \Wdim\}$. The \emph{deterministic assignment} of an independent distribution $p_\probs$ is the partial assignment $\implicant=\{\varprob_i\}_{i\in D}$ defined by its deterministic factors $D=\{i|\varprob_i\in \{0, 1\}\}$. %
 \end{definition}

 \begin{definition}[Implicants]
	\label{def:implicant}
	We define the \emph{cover} $\mathcal{W}_\implicant\subseteq \{0, 1\}^n$ of a partial assignment $\implicant$ as the set that contains all worlds $\bw\in \{0, 1\}^n$ that equal $\implicant$ on the variables in $D$.
	A partial assignment $\implicant$ is an \emph{implicant} of $\knowledge{\by}$ if its cover only contains possible worlds. That is, $\implicant \models \knowledge{\by}$.
\end{definition}

Intuitively, an implicant assigns values to a subset of the variables in $\bw$ such that it ensures the constraint $\knowledge{\by}$ is true. For the traffic lights example, the partial assignment $\neg r$ is an implicant of $\knowledge{\by}$, since its cover ($\neg r \wedge g$ and $\neg r \wedge \neg g$) only contains possible worlds.
Our first result states that if we have an implicant, we can easily create possible independent distributions: Use the implicant as the deterministic part, and assign any value in $[0, 1]$ to the remaining factors. This is because, for implicants, the value of the other variables \say{does not matter} to the constraint $\knowledge{\by}$.

\begin{theorem}[Implicants determine possible independent distributions]
	\label{thm:independence}
	Let $p_\probs$ be an independent distribution over worlds. Let $\implicant$ be $p_\probs$'s deterministic assignment. Then $p_\probs$ is possible for $\knowledge{\by}$ if and only if  $\implicant$ is an implicant of $\knowledge{\by}$.
\end{theorem}
\begin{proof}
	By independence of $p_\probs$, the support of $p_\probs$ is the cover $\mathcal{W}_\implicant$ of the deterministic assignment $\implicant$ of $\probs$, and the remaining variables can be assigned any value with some probability.
	Assume $p_\probs$ is possible; then, for each $\bw$ in the support of $p_\probs$, $\bw$ is a possible world. But then each world in the cover $\mathcal{W}_\implicant$ of $\implicant$ is possible, and so $\implicant$ is an implicant.
	Next, assume $\implicant$ is an implicant. Then, each world in the cover of $\implicant$ is possible. But this is precisely the support of $p_\probs$. So $p_\probs$ is possible. \qedhere

\end{proof}
The more restrictive the constraint is over what worlds are possible, the more variables the implicants assign values to. Our example contains five implicants: $\neg r \wedge g$, $r \wedge \neg g$, $\neg r \wedge \neg g$, $\neg r$, and $\neg g$. Therefore, the deterministic assignment of $p_\probs$ will need to contain at least one of $\neg r$ and $\neg g$ for $p_\probs$ to be possible. %

\subsection{Conditioned independent distributions}
A common counterargument to the claim that independent distributions are biased towards determinism is that we can condition an independent distribution $p_\probs$ on the constraint $\knowledge{\by}$. This is the distribution $p_\probs(\bw|\knowledgevar)$ where the variables $w_i$ become dependent due to conditioning on the constraint. Furthermore, such a parameterisation is an $n$-dimensional manifold inside the set of possible distributions, which can cover far more distributions than those characterised in Theorem \ref{thm:independence}. For instance, consider the independent distribution $p(r)=p(g)=\frac{1}{2}$. When conditioning on $\varphi=\neg r \vee \neg g$, we get the uniform distribution over worlds $p_1$ from Figure \ref{fig:intro-fig}. In fact, we can cover the entire green triangle in Figure \ref{fig:intro-fig} in this way.

However, this argument does not hold in a \emph{learning} setting where we optimise towards a minimum of the semantic loss (Equation \ref{eq:wmc}). As we already showed, the uniform distribution over worlds $p_1$ can \emph{not} be represented by an independent distribution alone. In fact, there are strict conditions on when an independent distribution $p_\probs$ conditioned on the constraint $\knowledge{\by}$ can be represented by \emph{another} independent distribution $q_{\probs'}$:

\begin{theoremrep}[Representability of $p_\probs(\bw|\knowledgevar)$]
	\label{prop:independent-to-independent-PR}
	Let $p_\probs$ 
 have $p_\probs(\knowledgevar)>0$ 
    and deterministic assignment $\altimplicant$. Then the following statements are equivalent:%
	\begin{enumerate}
		\item the conditional distribution $p_\probs(\bw|\knowledgevar)$ can be represented by another independent distribution $q_{\probs'}$;
		\item there is an implicant $\implicant$ that covers all possible worlds in the support of $p_\probs$;
		\item there is an implicant $\implicant$ such that $\altimplicant, \knowledge{\by}\models \implicant$.
	\end{enumerate}
\end{theoremrep}
\begin{proofsketch}
	Under this condition, we can rewrite the conditional distribution $p_\probs(\bw|\knowledgevar)$ as an independent distribution where the deterministic assignment is $\implicant$, and the remaining factors are renormalised to sum to 1. 
\end{proofsketch}
\begin{appendixproof}
	$2\rightarrow 1$: Assume such an implicant $\implicant$ of $\knowledge{\by}$ exists. Then we rewrite the conditional distribution $p_\probs(\bw|\knowledgevar)$ as %
	\begin{align}
		p_\probs(\bw|\knowledgevar)&=\frac{\prod_{i=1}^\Wdim \prob_i \knowledge{\by}(\bw)}{p_\probs(\knowledgevar)}\\
		&= I[\bw\in \mathcal{W}_\implicant] \frac{\prod_{i\not \in D} \prob_i }{p_\probs(\knowledgevar)} \\
		&= I[\bw\in \mathcal{W}_\implicant] \prod_{i\not \in D} \prob_i p_\probs(\knowledgevar)^{\Wdim - |D|}.
	\end{align} 
	Here, $I[\bw\in \mathcal{W}_\implicant]$ is the indicator function that is 1 when $\bw$ is in the cover of $\implicant$. 
	This is an independent distribution $q_{\probs'}$ with deterministic assignment $\implicant$ and parameters $\prob_i'=\prob_i p_\probs(\knowledgevar)^{\Wdim - |D|}$ for the stochastic variables. In the first step, we used that $\knowledge{\by}(\bw)=0$ exactly when $\bw_D$ differs from the implicant $\implicant$.

	$1\rightarrow 2$: Assume there is no implicant $\implicant$ as described in 2. Then, the deterministic assignment of the conditional distribution $p_\probs(\bw|\knowledgevar)$ is not an implicant, as if it was, we could have constructed such an implicant. But then, by Theorem \ref{thm:independence}, there must be a world $\bw$ in the cover of the deterministic assignment of $q$ that is not possible. Since for independent distributions, any worlds in the cover of the deterministic assignment get positive probability, such an independent distribution must also assign positive probability to this extension, yet $p_\probs(\bw|\knowledgevar)=0$, so $p_\probs(\bw|\knowledgevar)$ cannot be represented by an independent distribution.  %

	$2\rightarrow 3$: Assume all possible worlds in the support of $p_\probs$ are in the cover of the implicant $\implicant$. That means all worlds in the cover of $\altimplicant$ for which the constraint $\knowledge{\by}$ holds are also in the cover of $\implicant$. Therefore, $\altimplicant, \knowledge{\by}\models \implicant$.

	$3\rightarrow 2$: Assume there is an implicant $\implicant$ such that $\altimplicant, \knowledge{\by}\models \implicant$. By independence, the support of $p$ contains the worlds extending $\altimplicant$. By the entailment, its subset of \emph{possible} worlds is those that also extend $\implicant$. %
\end{appendixproof}

This theorem is rather subtle. In our example, an independent distribution with deterministic assignment $g$ \say{entails} $\neg r$: Since $\neg g$ is not true, we need to make $\neg r$ true to be consistent with $\neg r \vee \neg g$. Since $\neg r$ is an implicant, an independent distribution can represent $p_\probs(\bw|\knowledgevar)$. Any distribution with $p_\probs(g)=1$ will have a conditional distribution at the vertex $(\neg r, g)$, which is representable. Therefore, the distributions for which we can represent $p_\probs(\bw|\knowledgevar)$ are those that either have $p_\probs(g)=1$ or $p_\probs(r)=1$ (but not both, since then $p_\probs(\knowledgevar)=0$).
For general functions, this theorem states that a distribution can only represent a conditional independent distribution if the unconditioned distribution is deterministic in some variables.

\subsection{The geometry of sets of possible independent distributions}
\label{sec:geometric}
While the previous section studied possible independent distributions individually, in the following section we study entire sets of possible distributions from a geometric and topological viewpoint. Our main result proves that all possible independent distributions are on a face of the hypercube $[0, 1]^n$ and that we can compute which faces contain possible independent distributions. 

We next define the \emph{set of all independent distributions} $\independent \subseteq \distributions$. We calculate the probability of each world from the parameters $\probs\in [0,1]^n$. Then, we create a vector in the set of all distributions $\distributions$: 
\begin{equation}
	\begin{aligned}
	\label{eq:independentset}
	\independent&= \Big\{p_\probs \in \distributions \mid \probs\in [0, 1]^n, \\
	&{p_\probs}_i =  \prod_{j=1}^\Wdim \varprob_j^{w_{i, j}} \cdot (1-\varprob_j)^{1 - w_{i, j}}  \Big\}
	\end{aligned}
\end{equation}
We consider all parameters of possible distributions in $[0, 1]^n$. Then, we compute the probability of each world $\bw_i$ using Equation \ref{eq:indep-assumption} and the Bernoulli mass function to create a vector in $\distributions$. This map from $[0, 1]^n$ to $\distributions$ is a bijection (see Lemma \ref{lemma:bijection}), and so is a homeomorphism between the $n$-cube $[0, 1]^n$ and $\independent$. In practice, this means we can study topological properties both in parameter space and distribution space. We will treat $p_\probs$ as both a vector in $\distributions$ and a distribution over worlds. 

Independent distributions can only represent a subset of the simplex $\distributions$. We aim to understand this subset, and in particular the set of \emph{possible independent distributions} $\possindependent=\independent \cap\poss{\by}$.

\subsubsection{A representation of $\possindependent$}
\label{sec:representation}
We next prove that the set of possible independent distributions $\possindependent$ is formed by considering the set of all prime implicants of $\knowledge{\by}$. 
We find a useful representation of $\possindependent$ using \emph{cubical sets} (often called \emph{cubical complexes}). %
\citet{rothAlgebraicTopologicalMethods1958} was the first to use cubical sets to develop algorithms that compute efficient representations of boolean functions, noting the relation to implicants. 
Intuitively, a cubical set is a union of (hyper)cubes of various dimensions. In our representation, we use implicants to create a cube. We then show a cubical set formed from such cubes is the set of possible independent distributions. 

The cube associated with an implicant fixes the coordinates of the deterministic variables and uses the interval $[0, 1]$ for the free variables. For example, the implicants of the traffic light problem form two cubes: For $\neg r$, the cube is $\{0\}\times [0, 1]$ (or: the first is false, and the second is \say{agnostic}) and for $\neg g$, the cube is $[0, 1]\times \{0\}$. We next discuss the relevant background for these concepts.
\begin{definition}[Elementary cubes]
	\label{def:cube}
	An \emph{elementary interval} $I$ is $\{0\}$,  $\{1\}$, or $[0, 1]$. An \emph{(elementary) $n$-cube} $C$ is the Cartesian product of $n$ elementary intervals $C=I_1 \times \cdots \times I_n\subseteq [0, 1]^n$. We use \say{cube} to refer only to elementary cubes unless mentioned otherwise. The \emph{dimension} of a cube is the number of elementary intervals $I_i$ that are $[0, 1]$. Cubes of dimension 0 are called \emph{vertices} and are points in $\{0, 1\}^n$, while cubes of dimension 1 are \emph{edges} that connect two vertices.
	A \emph{face} $C'$ of a cube $C$ is a cube such that $C'\subseteq C$. 
 \end{definition}
 
\begin{definition}[Cubical sets]	
	\label{def:cubical}
	$X\subseteq [0, 1]^n$ is a \emph{cubical set} if it is the finite union of a set of cubes $\{C_1, ..., C_k\}$ \citep{kaczynskiComputationalHomology2004}. With $\mathcal{C}(X)$ we denote all faces of the cubes $\{C_1, ..., C_k\}$, while with $\mathcal{C}_k$ we denote the faces in $\mathcal{C}(X)$ of dimension $k$, called the \emph{$k$-cubes} of $X$. A \emph{facet} $C\in \mathcal{C}(X)$ of $X$ is a cube that is not contained in another cube $C'\in \mathcal{C}(X)$.  %
\end{definition}

Next, we need the notion of prime implicants \citep{quineCoresPrimeImplicants1959}. Informally, an implicant is a prime implicant if, by removing any of its assignments, there will be extensions of the implicant that are impossible worlds.  
\begin{definition}[Prime implicant]
	\label{def:prime-implicant}
	An implicant $\implicant$ of $\knowledge{\by}$ is a \emph{prime implicant} if its cover $\mathcal{W}_\implicant$ is not contained in the cover $\mathcal{W}_\altimplicant$ of another implicant $\altimplicant$, that is, $\mathcal{W}_\implicant \not\subset \mathcal{W}_\altimplicant$ for all implicants $\altimplicant$. With $\mathcal{I}=\{\bw_{D_i}\}_{i=1}^m$ we denote the set of all prime implicants of $\knowledge{\by}$.
\end{definition}
The set of prime implicants $\mathcal{I}$ can be found with the first step of the  Quine--McCluskey algorithm \citep{quineProblemSimplifyingTruth1952,mccluskeyMinimizationBooleanFunctions1956}. It creates a disjunctive normal form of $\knowledge{\by}$ by considering their disjunction. 

Finally, we introduce the cubical set corresponding to $\knowledge{\by}$.
\begin{definition}[Implicant cubes \& cubical set of $\knowledge{\by}$]
	\label{def:implicant-cube}
	Each implicant $\implicant$ defines an \emph{implicant cube} $\Pimplicant$: Its $i$-th elementary interval $I_i$ is $\{\implicant_i\}$ if $i\in D$, and $[0, 1]$ otherwise. The \emph{cubical set $\cubepossible$ of $\knowledge{\by}$} is the union of all prime implicant cubes $\cubepossible = \bigcup_{\implicant \in \mathcal{I}}\Pimplicant$.
\end{definition}
\begin{example}
	In the traffic light problem, the prime implicants are $\neg r$  and $\neg g$. $\neg r \wedge \neg g$ is an implicant but is not prime, as two proper subsets are also implicants.

	The cubical set is $\cubepossible=\{0\} \times [0, 1] \cup [0, 1] \times \{0\}$. $\mathcal{C}_0(\cubepossible)= \{(0, 0), (0, 1), (1, 0) \}$ is the vertices, while $\mathcal{C}_1(\cubepossible)=\{\{0 \} \times [0, 1], [0, 1] \times \{0\} \}$ are the edges. The first edge connects $(0, 0)$ and $(0, 1)$, and the second connects $(0, 0)$ and $(1, 0)$. This cubical set corresponds to the lines of minimal loss in the left plot of Figure \ref{fig:traffic-lights}.
\end{example}

The implicant cube $\Pimplicant$ contains the independent parameters $\probs$ for distributions $p_\probs$ that deterministically return $\implicant$. We present the basic properties of $\cubepossible$ in Appendix \ref{appendix:cubical}. The most important results are that the set of cubes $\mathcal{C}(\cubepossible)$ is equal to the set of all implicant cubes. Furthermore, the \emph{prime} implicant cubes are the facets of the cubical set $\cubepossible$. 
This means we can exactly compute the combinatorial structure of $\cubepossible$ from the prime implicants of $\knowledge{\by}$.

Our next result states that the cubical set $\cubepossible$ indeed represents the set of possible independent distributions.
\begin{theoremrep}[Representing the set of possible independent distributions]
	\label{prop:union}
	A parameter $\probs$ is in $\cubepossible$ if and only if the distribution $p_\probs$ is possible for $\knowledge{\by}$. That is, $\probs\in \cubepossible$ if and only if $p_\probs\in \possindependent$.
	Furthermore, the cubical set $\cubepossible$ cannot be represented as a union of fewer cubes.
\end{theoremrep}
\begin{proofsketch}
	A distribution $p_\probs$ using a parameter $\probs$ from implicant cube $\Pimplicant$ fixes $\implicant$ and allows any value in $[0, 1]$ for the remaining factors. This describes all possible independent distributions by Theorem \ref{thm:independence}, so all possible independent distributions are in the union of implicant cubes.

	Next, we show that a parameter $\probs$ that is in the open interval $(0, 1)$ for all stochastic variables of a prime implicant cube cannot be in another (prime) implicant cube. This is because $\mu$ would be in the \emph{relative interior} of $\Pimplicant$, and we know that the relative interiors of faces of a cubical set are disjoint \citep{kaczynskiComputationalHomology2004}. This shows that no smaller set of implicants gets us to $\possindependent$. 
\end{proofsketch}
\begin{appendixproof}
	Consider some $p_\probs\in \possindependent$. Then, by Theorem \ref{thm:independence}, the deterministic part $\altimplicant$ of $\probs$ is an implicant of $\knowledge{\by}$. Let $\implicant\in \mathcal{I}$ be any prime implicant that $\altimplicant$ is an extension of, which has to exist by construction. Clearly $\Paltimplicant \subseteq \Pimplicant$, and so $\probs \in \Pimplicant \subseteq \cubepossible$.
	Next, assume $\probs\in \cubepossible$ for some $\implicant \in \mathcal{I}$. $\implicant$ is an implicant of $\knowledge{\by}$, and so by Theorem \ref{thm:independence}, $p_\probs$ is possible for $\knowledge{\by}$, hence $p_\probs\in \possindependent$.

	Next, we prove that there is no smaller set of cubes than the set of prime implicant cubes that generate $\cubepossible$. %
	Associated with each implicant cube $\Pimplicant$, is an \emph{(elementary) implicant cell} $\cellimplicant\subseteq \Pimplicant$. $\cellimplicant$ is similarly defined as $\Pimplicant$, except it uses open intervals $(0, 1)$ instead of closed intervals $[0, 1]$. 
	Let $\implicant$ be a prime implicant and let $\probs \in \cellimplicant$. All cubes are equal to the union of cells inside it (\citet{kaczynskiComputationalHomology2004}, Proposition 2.15(v)). Since the different cells in a cubical complex are disjoint (\citet{kaczynskiComputationalHomology2004}, Proposition 2.15(iii)), $\probs$ is not in another implicant cube. 
\end{appendixproof}
While we can compute this representation, it can also be rather big. The number of prime implicants can be exponential in the number of variables, and there are formulas with $\Omega(3^n/n)$ prime implicants \citep{chandraNumberPrimeImplicants1978}, above the number of worlds $2^n$. And as we proved, we need \emph{all} prime implicants: A minimal subset of prime implicants that cover all possible worlds (for instance, the prime implicants found in the second step of the Quine--McCluskey method) does not always cover $\possindependent$. See Appendix \ref{appendix:minimal-cover} for a counterexample.
In addition, computing the combinatorial structure of the cubical set $\cubepossible$ adds a significant combinatorial overhead, as it is generated from the prime implicants.

\subsubsection{Convexity of semantic loss over independent distributions}
\label{sec:convexity}
Next, we study when the semantic loss restricted to independent distributions is convex.  We already saw in Figure \ref{fig:traffic-lights} that even for the simple traffic light formula, the set of possible independent distributions is not convex. We now prove this is almost always the case:
\begin{theoremrep}[Convexity]
	\label{thm:convex}
	The following statements are equivalent: 1) There is exactly one prime implicant of $\knowledge{\by}$;  2) $\cubepossible$ %
	is convex; 3) the semantic loss over the space of independent distributions $\mathcal{L}(\probs)$ is convex.
\end{theoremrep}
\begin{proofsketch}
	If there is exactly one prime implicant, $\cubepossible$ is an implicant cube $\Pimplicant$, which is clearly convex. If there is more than one prime implicant, we can construct a convex combination of two parameters that is not possible by noting that the deterministic assignment of this convex combination is not an implicant.

	The convexity of the semantic loss is proven using Jensen's inequality and noting that we can marginalise out all the stochastic variables. With more than one prime implicant, we note that since its minima $\cubepossible$ are non-convex, certainly the semantic loss must also be non-convex.
\end{proofsketch}
\begin{appendixproof}
	$1\leftrightarrow 2$ follows directly from \citet{kaczynskiComputationalHomology2004}, Proposition 2.80, by noting that the only rectangles in our setting are the elementary cubes in $[0, 1]$. Since there is no proof of Proposition 2.80 given in \citet{kaczynskiComputationalHomology2004}, we provide it for completeness’ sake. 

	$2\rightarrow 1$: Assume there is exactly one prime implicant $\implicant$ of $\knowledge{\by}$. Then $\cubepossible$ is described by $\Pimplicant$. This is an elementary cube, which is convex. 

	$1\rightarrow 2$: Next, assume there is more than one prime implicant. Consider two distinct prime implicants $\implicant, \altimplicant\in \mathcal{I}$. Consider $\probs_{\implicant} \in \Pimplicant \setminus \Paltimplicant$ and $\probs_{\altimplicant}\in \Paltimplicant\setminus \Pimplicant$, which have to exist by Proposition \ref{prop:union}. Consider $\probs$ to be any non-trivial convex combination of $\probs_\implicant$ and $\probs_\altimplicant$. Note that the deterministic assignment of $\probs$ is $\hat{D}=\{k\in D \cap E: \implicant_{k}=\altimplicant_{k}\}$. Since the prime implicants are different, at least one of the following needs to hold:
	\begin{enumerate}
		\item There is a $k\in D$ but $k\not\in E$. %
		Then $\varprob_k$ is stochastic, and $\hat{D}$ is a strict subset of $D$.  
		\item There is a $k\in D \cap E$ such that $\implicant_{k}\neq \altimplicant_{k}$. By the convex combination, $\varprob_k$ is stochastic, as it assigns probability mass to both $\implicant_k$ and $\altimplicant_k$. Therefore, $\hat{D}$ is a strict subset of $D$. %
	\end{enumerate}
	Since $D$ is a prime implicant, removing any element from $D$ results in a deterministic assignment that is no longer an implicant. Thus, by Theorem \ref{thm:independence}, $p_\probs$ is not possible. 

 \allowdisplaybreaks
	
	$2\rightarrow 3$: Assume there is exactly one prime implicant $\implicant$. Then note that the only possible worlds are the cover $\mathcal{W}_\implicant$. The cover contains all assignments to the stochastic variables, i.e., the variables not in $D$. This means we can safely marginalize those out: 
	\begin{align*}
		p_\probs(\knowledgevar) 
                =& \sum_{\bw \in \mathcal{W}_\implicant }  p_\probs(\bw) \\
                =& \sum_{\bw \in \mathcal{W}_\implicant } \prod_{i\in D} p_\probs(\implicant_i) \prod_{i\in\{1, ..., n\} \setminus D} p_\probs(\bw_i)\\
                =&\prod_{i\in D} p_\probs(\implicant_i)\sum_{\bw \in \mathcal{W}_\implicant }   \prod_{i\in\{1, ..., n\} \setminus D} p_\probs(\bw_i)\\
                =& \prod_{i\in D} p_\probs(\implicant_i)\sum_{\bw \in \mathcal{W}_\implicant }   p_\probs(\bw_{\{1, ..., n\} \setminus D}) \\
                =&\prod_{i\in D} p_\probs(\implicant_i)=p_\probs(\implicant).
	\end{align*}
	Given $\lambda\in (0, 1)$, $\probs_1, \probs_2\in [0,1]^n$, we define $\probs_\lambda=\lambda \probs_1 + (1-\lambda) \probs_2$ for brevity. Rewriting,  and using Jensen's inequality and some slightly laborious algebra, we find that
	\begin{align*}
		&\mathcal{L}(\probs_\lambda)= -\log p_{\probs_\lambda}(\knowledgevar) = -\log p_{\probs_\lambda}(\implicant) \\
		=& -\log \prod_{i\in D} p_{\prob_\lambda}({\bw_{D}}_i)\\
		=& -\sum_{i\in D} \log( {\prob_\lambda}_i^{\implicant_i} (1-{\prob_\lambda}_i)^{1-\implicant_i}) \\
		=& -\sum_{i\in D} \implicant_i \log {\prob_\lambda}_i + (1-\implicant_i)\log(1-{\prob_\lambda}_i)\\
		=&-\sum_{i\in D} \implicant_i \log (\lambda {\prob_1}_i + (1-\lambda){\prob_2}_i) \\
		&+ (1-\implicant_i)\log(1-(\lambda {\prob_1}_i + (1-\lambda){\prob_2}_i))\\
		=&-\sum_{i\in D} \implicant_i \log (\lambda {\prob_1}_i + (1-\lambda){\prob_2}_i) \\
		&+ (1-\implicant_i)\log(\lambda (1-{\prob_1}_i) + (1-\lambda)(1 - {\prob_2}_i)) \\
		\leq& -\sum_{i\in D} \implicant_i (\lambda \log {\prob_1}_i + (1-\lambda)\log {\prob_2}_i)\\ 
		& + (1-\implicant_i)(\lambda \log(1-{\prob_1}_i) + (1-\lambda)\log (1 - {\prob_2}_i))\\
		=&- \sum_{i\in D} \lambda \log {\prob_1}_i^{\implicant_i}(1-{\prob_1}_i)^{1-\implicant_i} \\
		&+ (1-\lambda)\log {\prob_2}_i^{\implicant_i}(1-{\prob_2}_i)^{1-\implicant_i}\\
		=&-\lambda \log p_{\probs_1}(\knowledgevar) - (1-\lambda)\log p_{\probs_2}(\knowledge{\by}) \\
		=& \lambda \mathcal{L}(\probs_1) + (1-\lambda)\mathcal{L}(\probs_2).
	\end{align*}
	
	$3\rightarrow 2$. Assume there is more than one prime implicant. Using $1\rightarrow 2$, this means $\cubepossible$ is non-convex. Therefore, there is a pair $\probs_1, \probs_2\in \cubepossible$ and $\lambda\in (0, 1)$ such that $\probs_\lambda=\lambda \probs_1 + (1-\lambda) \probs_2\not\in \cubepossible$. By Theorem \ref{prop:union}, $\mathcal(\probs_1)=\mathcal{L}(\probs_2)=0 < \mathcal{L}(\probs_\lambda)$, as $\cubepossible$ exactly describes the possible distributions where $\mathcal{L}(\probs)=0$. Therefore, $\mathcal{L}(\probs_\lambda) > \lambda \mathcal{L}(\probs_1) + (1-\lambda) \mathcal{L}(\probs_2)$, proving non-convexity. \qedhere

\end{appendixproof}
The condition that there is a single prime implicant means that the set of all possible worlds $\mathcal{W}_{\knowledge{\by}}$ is described by fixing some variables and letting the other variables be free. This is essentially \say{supervised learning} on the variables in $D$ and absolutely no supervision for the other variables. This is an uncommon scenario for most neurosymbolic settings, as we can simply resort to standard supervised learning methods. %

\subsubsection{Connectedness of $\possindependent$}
\label{sec:connectedness}
We next study when the set of all possible independent distributions $\possindependent$ is connected. For this, we introduce the notion of a \emph{prime implicant graph}:
\begin{definition}[Prime implicant graph]
	\label{def:prime-implicant-graph}
	Let $\mathcal{G}=(\mathcal{W}_{\knowledge{\by}}, \mathcal{E})$ be the \emph{prime implicant graph} of $\knowledge{\by}$, where the vertices $\mathcal{W}_{\knowledge{\by}}$ is the set of possible worlds of $\knowledge{\by}$ and $\mathcal{E}=\{(\bw_1, \bw_2)\mid \exists_{\implicant \in \mathcal{I}}: \bw_1, \bw_2 \in \mathcal{W}_\implicant \}$ is the set of edges. 
\end{definition}
In this graph, there is an edge between two possible worlds $\bw_1$ and $\bw_2$ when there is an implicant that covers both $\bw_1$ and $\bw_2$. In our traffic light example, the prime implicant graph has three vertices. There is an edge between the first and the third ($(0, 1)$ and $(0, 0)$) and the second and the third ($(1, 0)$ and $(0, 0)$). In the case of the XOR function (Appendix \ref{appendix:xor}) $(a \wedge \neg b) \vee (\neg a \wedge b)$, the graph has two vertices $(1, 0)$ and $(0, 1)$, but no edges.

\begin{theoremrep}[Connectedness]
	\label{thm:connected}
	The connected components of the space of possible distributions $\cubepossible$ correspond to the connected components in $\mathcal{G}$. %
	In particular, $\cubepossible$ is a connected space if and only if $\mathcal{G}$ is connected.
\end{theoremrep}
\begin{proofsketch}
	The vertices of the prime implicant graph $\mathcal{G}$ as points in $\{0, 1\}^n$ directly correspond to the vertices of $\cubepossible$. When an edge exists between $\bw_1$ and $\bw_2$, both worlds are covered by some prime implicant $\implicant$, and their deterministic components must include $\implicant$, so $\bw_1, \bw_2\in \Pimplicant$. Since all cubes are connected, $\bw_1$ and $\bw_2$ are connected in $\Pimplicant\subseteq\cubepossible$. By induction, if there is a path in $\mathcal{G}$ between $\bw_1$ and $\bw_2$, there is a path in $\cubepossible$ between $\bw_1$ and $\bw_2$.
\end{proofsketch}
\begin{appendixproof}
	By Proposition \ref{prop:vertices_are_worlds}, the vertices of $\cubepossible$ are precisely the possible worlds $\mathcal{W}_{\knowledge{\by}}$. We say $\bw_1, ..., \bw_n\in \mathcal{W}_{\knowledge{\by}}$ are \emph{edge-connected} if there exist edges $E_1, ..., E_n\in \mathcal{C}_1(\cubepossible)$ such that $\bw_i$ and $\bw_{i+1}$ are faces of $E_i$. 
	
	Consider an edge $(\bw_1, \bw_2)\in \mathcal{E}$. Then there is a prime implicant $\implicant$ that covers both $\bw_1$ and $\bw_2$. Therefore, $\bw_1$ and $\bw_2$ are both in $\Pimplicant$. Since $\Pimplicant$ is an (elementary) cube, by \citet{kaczynskiComputationalHomology2004}, Proposition 2.51.1, $\bw_1$ and $\bw_2$ are edge-connected. 
	
	Let $E\in \mathcal{C}_1(\cubepossible)$ be an edge with vertices $\bw_1$ and $\bw_2$. $E$ is a face of a prime implicant cube $\Pimplicant$ by Definition \ref{def:cubical}, and by Proposition \ref{prop:vertices_are_covers}, $\bw_1$ and $\bw_2$ are both in the cover $\mathcal{W}_\implicant$. Therefore, $(\bw_1, \bw_2)\in \mathcal{E}$.

	Combining these two results, we find that $\bw_1$ and $\bw_2$ are edge-connected if and only if there is an edge between them in $\mathcal{G}$. Therefore, by Theorem 2.55 and Corollary 2.57 of \citet{kaczynskiComputationalHomology2004}, the connected components of $\mathcal{G}$ and $\cubepossible$ coincide. \qedhere

\end{appendixproof}
This theorem shows that the connectedness depends on the structure of the constraint. For the traffic light example, $\possindependent$ is connected: the three possible worlds are connected through the two prime implicants $\neg r$ and $\neg g$. However, $\possindependent$ is disconnected for the XOR function, as its prime implicant graph is disconnected. See Appendix~\ref{appendix:xor} for a visualisation. The popular MNIST Addition task (Example \ref{ex:mnist-add}) is another example: Like XOR, $\possindependent$ is a set of disconnected vertices. This brings challenges to training independent perception models: Each disconnected part of the graph is a different \say{global optimum}, and moving from one global optimum to another will require a large change in parameters while incurring a higher loss.

\subsection{Computing the homology of $\possindependent$}
\label{sec:homology}
Our representation of the set of possible independent distributions is in the form of a cubical set. Cubical sets are a useful representation tool for topological spaces in algebraic topology, although simplicial complexes are more common \citep{hatcherAlgebraicTopology2002,matousekUsingBorsukUlam2008}. Homology allows us to use combinatorial, algebraic objects to study the topology of cubical sets. 
In our case, these objects correspond to the implicants of the formula. For finite cubical sets, the problem of computing the homology is solved \citep{kaczynskiComputationalHomology2004}. This means we can associate every formula $\knowledge{\by}$ with a homology which gives all the \emph{holes} in the set. This roughly tells us how \say{easy} this set is to traverse during optimisation: A set with many holes will require more complicated paths. We give an example of a formula with a hole in Appendix~\ref{sec:holes}.

The algorithm behind computing the homology from a cubical set is complicated, and involves both (abelian) group theory and linear algebra. We refer the reader to Algorithm 3.78 in \citet{kaczynskiComputationalHomology2004}, which requires the facets of the cubical set as input. In our case, this corresponds to the set of all prime implicants, found by the Quine--McCluskey algorithm \citep{quineProblemSimplifyingTruth1952}. Then we construct matrices that correspond to the boundaries of the cubes, and use linear algebra to compute the Smith normal form. From this matrix, the relevant groups can be constructed.

\section{Empirical visualisations}
\label{sec:empirical}
\begin{figure}
    \centering
    \includegraphics[width=\linewidth]{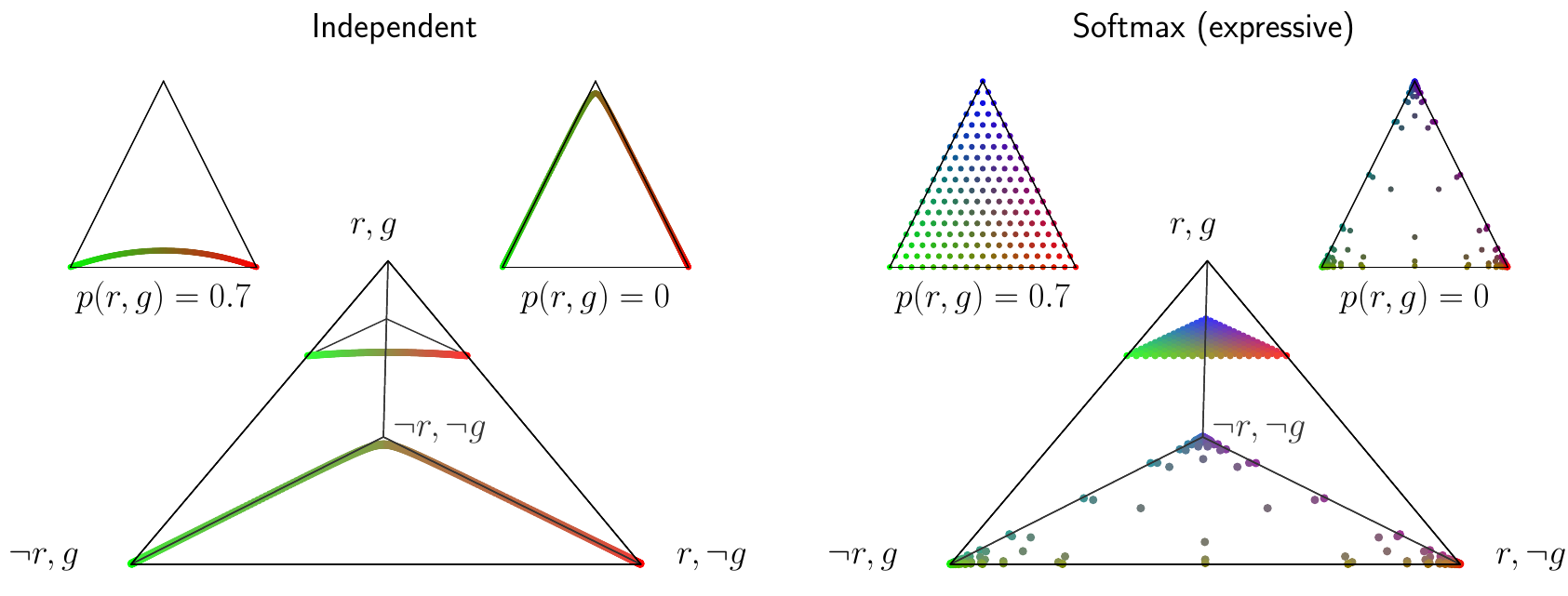}
    \caption{The minimisation of the semantic loss on the traffic light problem for independent distributions (left) and expressive distributions (right). The initial distributions have impossible beliefs with $p(r, g)=0.7$, plotted in the top-left triangle and the top triangle within the tetrahedron.  The resulting minima with $p(r, g)=0$ are in the bottom triangle. Minima of the independent assumption are as predicted by Theorem \ref{prop:union}. The minima of the expressive parameterisation cover differing areas in the bottom triangle, but are close to the vertices.}
    \label{fig:empirical-vis}
\end{figure}
To visualise what the possible distributions found by minimising the semantic loss look like, we compare independent distributions and expressive distributions on the traffic light problem in Figure \ref{fig:empirical-vis}. We modelled independent distributions with two real-valued parameters and a sigmoid, and expressive distributions with 4 real-valued parameters and a softmax. Then, we minimise the semantic loss to ensure $p(r, g)=0$. We use gradient descent for 10,000 iterations with a learning rate of 0.1. 

We find that the minima of the independent distributions are as predicted by Theorem \ref{prop:union}: A union of two line segments between the vertices $\neg r, g$ and $\neg r, \neg g$ and the vertices $r, \neg g$ and $\neg r, \neg g$. 
The minima of the expressive distributions are more uniformly distributed over the bottom triangle, but are biased towards the vertices. Therefore, using an expressive parameterisation is not sufficient to ensure the model is calibrated, which is a common problem in neural networks \citep{pmlr-v70-guo17a}. In Appendix \ref{appendix:regularisation}, we also experimented with adding an entropy maximisation loss, which counteracts this bias. This is similar to how BEARS \citep{marconato2024bears} encourages diversity. 

\section{Related work}
\label{sec:related-work}

Many PNL systems use the independence assumption we discussed, such as semantic loss \citep{xuSemanticLossFunction2018}, DeepProbLog \citep{manhaeveDeepProbLogNeuralProbabilistic2018}, DeepStochLog \citep{wintersDeepstochlogNeuralStochastic2022}, A-NeSI \citep{vankriekenANeSIScalableApproximate2023}, NeurASP \citep{yangNeurASPEmbracingNeural2020}, and Scallop \citep{huangScallopProbabilisticDeductive2021}.
As previously mentioned, this makes probabilistic reasoning tractable, as computing \cref{eq:wmc} is a \#P-hard problem in general.
While they are not directly comparable, fuzzy methods for neurosymbolic learning also implicitly make this assumption \citep{serafiniLogicTensorNetworks2016,badreddineLogicTensorNetworks2022,vankriekenAnalyzingDifferentiableFuzzy2022}. We show that several fuzzy logics also bias towards determinism in Appendix \ref{appendix:t-conorms}, although more work is needed to prove that this happens with the same generality as for probabilistic methods.

However, several recent methods are also compatible with more expressive distributions and show significant accuracy improvements compared to methods relying on the independence assumption. The pseudo-semantic loss \citep{ahmedPseudosemanticLossAutoregressive2023} uses a pseudo-log-likelihood approximation to the semantic loss to train autoregressive perception models. NeuPSL \citep{pryorNeuPSLNeuralProbabilistic2023} uses energy-based models that can perform joint inference over multiple variables. In
semantic probabilistic layers, \citet{ahmedSemanticProbabilisticLayers2022} experiment with an alternative parameterisation that increases expressivity without losing tractability: namely, a \emph{mixture} of independent distributions \citep{vergariCompositionalAtlasTractable2021}. 
We study these mixtures thoroughly in Appendix \ref{appendix:mixture}.
In particular, we prove that using two mixture components ensures the minima are connected. However, to be able to mix between an arbitrary number of possible worlds, we need at least as many components as the size of a \emph{minimal} prime implicant cover. This is exponential in the number of variables in general. 
BEARS \citep{marconato2024bears} increases expressiveness by creating an ensemble of independent models, which has the same expressiveness guarantees as mixtures of independent distributions. BEARS explicitly uses this ensemble to increase uncertainty calibration. \citet{ceruttiHandlingEpistemicAleatory2022} consider a Bayesian approach for probabilistic circuits, overcoming the independence assumption to improve the estimation of uncertainty.

The study of the theory of neurosymbolic learning is still in its infancy. \citet{marconatoNotAllNeuroSymbolic2023} discuss \emph{Reasoning Shortcuts}, which are perception models that minimise the semantic loss, yet learn to predict worlds that are different from the ground truth. Since the independence assumption biases to determinism, we hypothesise that independent distributions are more likely than expressive models to converge to a single reasoning shortcut: They cannot properly express uncertainty between different reasoning shortcuts. Several recent papers study how to best deal with reasoning shortcuts \citep{marconato2024bears,li2023learning}. 

Furthermore, recent work has studied conditions for the learnability of the perception model \citep{wangLearningLatentModels2023} and error bounds on its generalisation gap \citep{wangRegularizationInferenceLabel2023}. 
The output layer of the perception model also affects expressivity. If it is low-rank, that is, the number of neurons is lower than the number of outputs, there is an additional decrease in expressivity known as the softmax bottleneck \citep{yang2018breaking} or the sigmoid bottleneck in the context of binary outputs \citep{grivas2024taming}. Our results, in particular Theorem \ref{thm:connected}, are also related to the \emph{connectivity barrier} in Monte Carlo approaches to neurosymbolic learning over independent models \citep{liSoftenedSymbolGrounding2023}. A future study into this relation may provide insights in how to speed up Monte Carlo methods in this setting.

\section{Conclusion}
We studied the independence assumption in neurosymbolic learning, which characterises several popular methods. We proved that this assumption biases neurosymbolic models towards deterministic solutions. As a result, they lack the ability to express uncertainty about multiple possibly valid options. We then used tools from logic and computational homology to study the structure of the set of possible independent distributions, and showed it is non-convex and disconnected in general. 

In future work, we want to study practical methods for expressive neurosymbolic learning that properly represent uncertainty about different valid worlds. Dropping the independence assumption means that inference becomes much more complex, so a thorough study of appropriate (approximate) inference methods is needed. Our theory can be extended to a thorough study of the trade-off between expressivity and tractability. Our analysis of the mixture of independent distributions in Appendix \ref{appendix:mixture} gives a stepping stone towards this goal. Another option is to consider constraints on continuous variables. Furthermore, a thorough study of the homology discussed in Section \ref{sec:homology} may reveal further insights into the topology of the set of possible distributions.

\section*{Impact statement}
This paper presents foundational work to understand and advance the field of neurosymbolic machine learning. As such, there could be many potential societal consequences of applications of our work, none of which, we feel, can be easily predicted and specifically highlighted here.
\section*{Acknowledgements}
Emile van Krieken was funded by ELIAI (The Edinburgh Laboratory for Integrated Artificial Intelligence), EPSRC (grant no. EP/W002876/1).
Pasquale Minervini was partially funded by ELIAI (The Edinburgh Laboratory for Integrated Artificial Intelligence), EPSRC (grant no.\ EP/W002876/1), an industry grant from Cisco, and a donation from Accenture LLP.
Antonio Vergari was supported by the \say{UNREAL: Unified Reasoning Layer for Trustworthy ML} project (EP/Y023838/1) selected by the ERC and funded by UKRI EPSRC.
We thank Emanuele Marconato, Andreas Grivas, Patrick Koopmann, Thiviyan Thanapalasingam, Javaloy, Nicola Branchini, Leander Kurscheidt, Frank van Harmelen, Annette ten Teije, Eleonora Giunchiglia, Alessandro Daniele, Samy Badreddine, Siegfried Nijssen, Stefano Teso, and Sagar Malhotra for productive discussions and feedback while writing this work. We also thank the anonymous reviewers for their valuable feedback.

\bibliographystyle{plainnat}
\bibliography{references.bib,references_extra.bib}

\newpage

\appendix

\begin{table*}[!th]
	\centering
\begin{tabular}{lll}
\toprule
	Symbol & Description & Reference \\ \hline
	$\bw$ & A world in $\{0, 1\}^\Wdim$ & Sec.~\ref{sec:background} \\
	$\bx$ & An high-dimensional input & Sec.~\ref{sec:background} \\
	$\btheta$ & Parameter of the perception model & Sec.~\ref{sec:background} \\
	$\knowledge{\by}$ & A constraint of $\{0, 1\}^\Wdim\rightarrow \{0, 1\}$ & Sec.~\ref{sec:background} \\
	$p_\btheta(\bw|\bx)$ & The perception model & Sec.~\ref{sec:background} \\
	$\mathcal{L}(\btheta)$ & The semantic loss & Sec.~\ref{sec:background} \\
	$\mathcal{W}$ & The set of all worlds $\{0, 1\}^{n}$ & Sec.~\ref{sec:independence} \\
	$\mathcal{W}_{\knowledge{\by}}$ & The set of possible worlds & Sec.~\ref{sec:independence} \\
	$\distributions$ & The set of all distributions over worlds & Sec.~\ref{sec:expressiveness} \\
	$\distributions_\knowledge{\by}$ & The set of possible distributions & Sec.~\ref{sec:expressiveness} \\
	$\independent$ & The set of independent distributions  & Sec.~\ref{sec:geometric} \\
	$\possindependent$ & The set of possible independent distributions & Sec.~\ref{sec:geometric} \\
	$\probs$ & A parameter for an independent distribution in $[0, 1]^n$ & Sec.~\ref{sec:possible-distributions} \\
	$p_\probs(\bw)$ & An independent distribution parameterised by $\probs$ & Sec.~\ref{sec:possible-distributions} \\
	$\implicant$ & A partial assignment to the variables in $D$. Usually an implicant & Def.~\ref{def:detassign} and \ref{def:implicant} \\
	$\mathcal{W}_\implicant$ & The cover of a partial assignment (all worlds that extend it) & Def.~\ref{def:implicant} \\
	$\mathcal{I}$ & The set of all prime implicants & Def.~\ref{def:prime-implicant} \\
	$C$ & An (elementary) cube & Def.~\ref{def:cube} \\
	$\mathcal{C}(X)$ & The set of all cubes in a cubical set $X$ & Def.~\ref{def:cubical} \\
	$C_\implicant$ & The cube corresponding to an implicant & Def.~\ref{def:implicant-cube} \\
	$C_{\knowledge{\by}}$ & The cubical set of parameters in $[0, 1]^\Wdim$ that satisfy the knowledge $\knowledge{\by}$ & Def.~\ref{def:implicant-cube} \\
	$\mathcal{G}$ & The prime implicant graph & Def.~\ref{def:prime-implicant-graph} \\
 \bottomrule
\end{tabular}
\caption{\textbf{Table of notation used in the paper}.
We use bold symbols $\bx$, $\bw$ to denote vectors, both real and boolean. We use $p$ and $q$, possibly parameterised, to refer to distributions over $\bw\in \{0, 1\}^n$. Since these correspond to vertices in $\distributions$ (the simplex over all possible worlds), we will treat $p$ as both a vector and a distribution.}
\label{tab:notation}
\end{table*}
\clearpage

\newpage
\clearpage
\section{Bias towards determinism for Fuzzy Logic}
\label{appendix:t-conorms}
\begin{figure}[!t]
	\includegraphics[width=\linewidth]{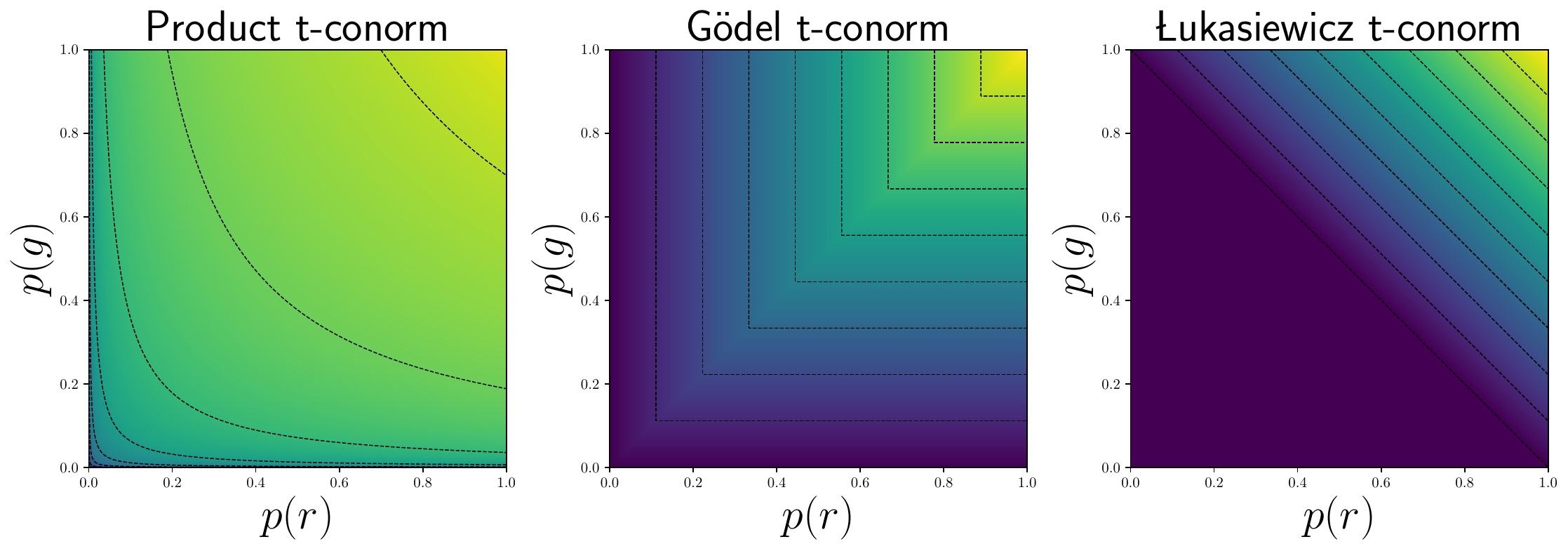}
	\caption{Plots of neurosymbolic loss functions for the formula $\neg r \vee \neg g$ using several t-norms. Left: Product t-norm, computed as $-\log p_\btheta(\knowledge{}|\bx)$. This coincides with the semantic loss. Center: The Gödel t-conorm $1 - \max(1-r, 1-g)$. Right: The \luk\ t-conorm $1 - \min(1, 2 - r-g)$.}
	\label{fig:traffic-lights}
\end{figure}

\textbf{Fuzzy Neurosymbolic Learning.} While our paper focuses on probabilistic methods, we shortly introduce relevant background about fuzzy neurosymbolic learning (FNL). Roughly, FNL methods construct a fuzzy evaluation function $e_\knowledge{\by}: [0, 1]^{\Wdim} \rightarrow [0, 1]$. $e_\knowledge{\by}$ maps independent probability distributions to fuzzy truth values in $[0, 1]$ by relaxing the logical connectives to operators on $[0, 1]$ \citep{badreddineLogicTensorNetworks2022}. If the distribution is deterministic, then the fuzzy truth becomes binary truth. For a discussion on fuzzy relaxations, see \citep{vankriekenAnalyzingDifferentiableFuzzy2022}. We limit our discussion to the three common fuzzy disjunctions (t-conorms):
\begin{align*}
	\text{Product: }  &a \vee b = a + b - a \cdot b, \\
	\text{Gödel: }  &a \vee b = \max(a, b), \\
	\text{\luk: } &a \vee b = \min(1,a + b)
\end{align*}
We plot the truth values of three common t-conorms for this formula in Figure \ref{fig:traffic-lights} as a function of $p(r)$ and $p(g)$. The product t-conorms and Gödel t-conorms have their minima at the lines $p(r)=0$ and $p(g)=0$, and have a similar biasing effect as the semantic loss. 

The \luk\ t-conorm is minimised when $p(r) + p(g)\leq 0.5$ and does not bias towards a deterministic choice. This may explain why \luk\ t-conorms are often more effective in realistic settings \citep{giunchigliaROADRAutonomousDriving2023}. However, the \luk\ logic has other problems, such as vanishing gradients \citep{vankriekenAnalyzingDifferentiableFuzzy2022} and the fact they do not converge to solutions where $p(\knowledgevar)=1$.

\section{Additional examples}
In this appendix, we plot several example formulas that illustrate the theory discussed in the paper.
\subsection{Minimal covers of prime implicants do not cover all possible independent distributions}
\label{appendix:minimal-cover}
\begin{figure}
	\centering
	\includegraphics[width=\linewidth]{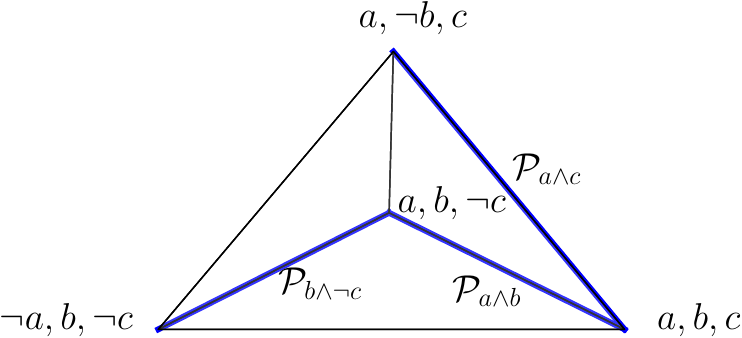}
	\caption{The full 3-simplex over possible worlds and the set of possible independent distributions in blue for the formula discussed in Section \ref{appendix:minimal-cover}. The $\mathcal{P}_{\phi}$ labels denote the set of distributions characterised by the implicant $\phi$, as defined in Proposition \ref{prop:union}.} 
	\label{fig:fourvars}
\end{figure}
In Proposition \ref{prop:union}, we showed that the set of \emph{all} prime implicants is necessary to cover all possible independent distributions. In this appendix, we give a counterexample to the idea that a \emph{minimal} cover of prime implicants might be sufficient to cover all possible independent distributions. Such minimal covers are computed in the second step of the Quine--McCluskey algorithm \citep{quineProblemSimplifyingTruth1952,mccluskeyMinimizationBooleanFunctions1956} to minimise the description length of the boolean formula. 

\begin{definition}
	A set of prime implicants $\mathcal{I}$ is a \emph{cover of $\knowledge{\by}$} if $\bigcup_{\implicant\in\mathcal{I}}\mathcal{W}_{\implicant}=\mathcal{W}_{\knowledge{\by}}$, that is, the union of their cover is equal to the set of all possible worlds. A cover is \emph{minimal} if no smaller set of prime implicants is also a cover.
\end{definition}

Consider a boolean formula on three variables with possible worlds $\{(a, b, c), (a, b, \neg c), (\neg a, b, \neg c), (a, \neg b, c)\}$. We visualize the full simplex over possible worlds and the set of possible independent distributions in Figure \ref{fig:fourvars}. The prime implicants of this formula are $\{b\wedge \neg c, a\wedge c, a\wedge b, \}$. The minimal cover of prime implicants is $\{b\wedge \neg c, a\wedge c\}$: The worlds $a\wedge b$ covers are $a\wedge b\wedge c$, which is also covered by prime implicant $a\wedge c$, and $a, b, \neg c$, which is also covered by prime implicant $b\wedge \neg c$. However, by Theorem \ref{thm:independence}, the distribution that deterministically assigns $a\wedge b$, but gives 0.5 probability to $c$, can be represented only using the prime implicant $a\wedge b$: The other two prime implicants cannot represent distributions where $c$ is stochastic. Therefore, minimal covers of prime implicants do not cover all possible independent distributions.

\subsection{The XOR formula has disconnected minima}
\label{appendix:xor}
\begin{figure}
	\centering
	\includegraphics[width=0.7\linewidth]{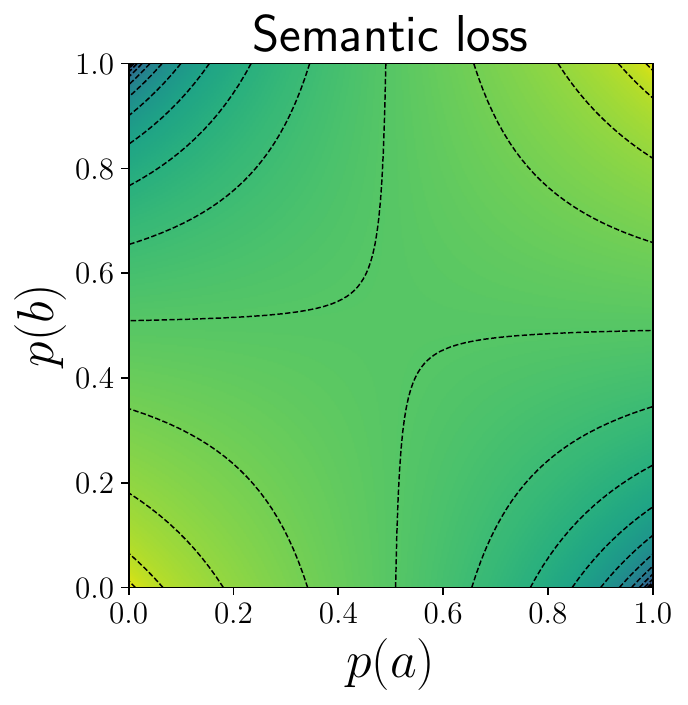}
	\caption{The loss landscape of the semantic loss under the independence assumption for the XOR formula $\varphi=(a\wedge \neg b) \vee (\neg a \wedge b)$.} 
	\label{fig:xor}
\end{figure}
In Figure \ref{fig:xor}, we plot the semantic loss under the independence assumption for the XOR formula $\varphi=(a\wedge \neg b) \vee (\neg a \wedge b)$. Note that the minima of this function are in $(1, 0)$ and $(0, 1)$, since the prime implicants are $a\wedge \neg b$ and $\neg a \wedge b$. These are clearly disconnected minima, and to move from one minimum to the other, we would have to traverse through the saddle point at $(0.5, 0.5)$, meanwhile incurring a significantly high loss.

\subsection{The set of possible independent distributions can have holes}
\label{sec:holes}
\begin{figure}
	\centering
\begin{tikzpicture}
        \draw[->] (-2,0.0,\lengthcube) -- (-1.3,0.0,\lengthcube) node[anchor=north east]{$\varprob_a$}; %
       \draw[->] (-2,0.0,\lengthcube) -- (-2,0.7,\lengthcube) node[below left]{$\varprob_c$}; %
       \draw[->] (-2,0.0,\lengthcube) -- (-2,0.0,\lengthcube-0.7) node[anchor=south]{$\varprob_b$}; %

        \draw (0, 0, 0)  -- (\lengthcube, 0, 0);
        \draw[blue,line width=0.25mm]  (\lengthcube,0,0) -- node[right] {$C_{a \wedge b}$} (\lengthcube,\lengthcube,0);
        \draw[blue,line width=0.25mm]  (0,\lengthcube,0) -- node[above] {$C_{b \wedge c}$} (\lengthcube,\lengthcube,0);
        \draw (0, 0, 0) --  (0, \lengthcube, 0);
        \draw (0,0,0) -- (0,0,\lengthcube) node[left] {$0, 0, 0$}; %
        \draw[blue,line width=0.25mm] (\lengthcube,0,0) -- node[right] {$C_{a\wedge \neg c}$} (\lengthcube,0,\lengthcube); %
        \draw (\lengthcube,\lengthcube,0)  node[right] {$1,1,1$} -- (\lengthcube,\lengthcube,\lengthcube); %
        \draw[blue,line width=0.25mm] (0,\lengthcube,0) -- node[left] {$C_{\neg a\wedge c}$} (0,\lengthcube,\lengthcube); %
        \draw[blue,line width=0.25mm] (0,0,\lengthcube) -- node[below] {$C_{\neg b\wedge \neg c}$} (\lengthcube,0,\lengthcube);
        \draw (\lengthcube,0,\lengthcube) -- (\lengthcube,\lengthcube,\lengthcube);
        \draw (\lengthcube,\lengthcube,\lengthcube) -- (0,\lengthcube,\lengthcube); 
        \draw[blue,line width=0.25mm] (0,\lengthcube,\lengthcube) -- node[left] {$C_{\neg a\wedge \neg b}$} (0,0,\lengthcube);
\end{tikzpicture}
\caption{An example of a formula where the set of possible independent distributions has a hole. The formula is $\knowledge{\by}=(\neg a \wedge \neg b) \vee (\neg a \wedge c) \vee (b\wedge c) \vee (a\wedge b) \vee (a\wedge \neg c) \vee (\neg b \wedge \neg c)$. The blue lines correspond to the set of possible independent distributions. $C_{\phi}$ is an implicant cube.}
\label{fig:hole}
\end{figure}
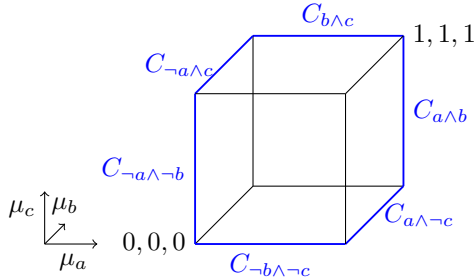

In Figure~\ref{fig:hole} we show that there are formulas for which the set of possible independent distributions $\possindependent$ has holes. Here, the formula $\knowledge{\by}=(\neg a \wedge \neg b) \vee (\neg a \wedge c) \vee (b\wedge c) \vee (a\wedge b) \vee (a\wedge \neg c) \vee (\neg b \wedge \neg c)$ is defined by a disjunction of prime implicants. We choose the prime implicants carefully so that each face of the cube has exactly 2 edges. The set highlighted in blue cannot be shrunk to a single point: It is a hole in space. This hole will be detected by algorithms that compute the homology of $\possindependent$ (see Section \ref{sec:homology}). How relevant is this to optimisation? The presence of holes means there is a cycle between different points, meaning we can move from one point to another in multiple ways. In our example, if one were to remove one of the prime implicants from the formula, there is only one path through $\possindependent$. 

\section{Entropy regularisation helps to calibrate expressive models}
\label{appendix:regularisation}
\begin{figure*}
	\centering
	\includegraphics[width=0.8\linewidth]{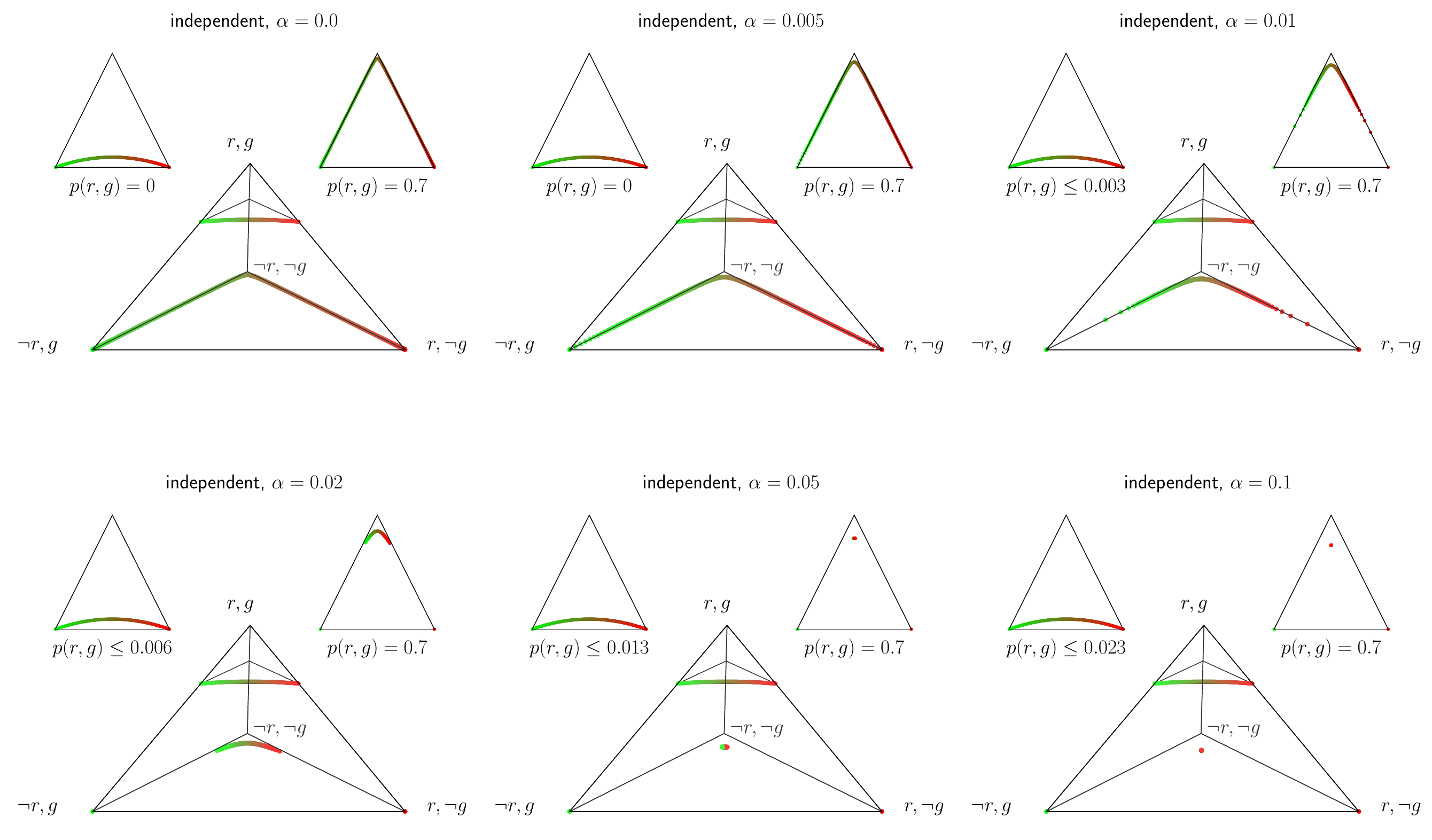}
	\caption{Minimising the semantic loss with entropy regularisation for the independent model.}
	\label{fig:entropy-independent}
\end{figure*}

We repeat the experiments in Section \ref{sec:empirical} for both the independent and the softmax model with entropy regularisation. We note that we \emph{maximise} entropy, instead of minimising the entropy, like in Neuro-Symbolic Entropy Regularisation \citep{ahmedNeuroSymbolicEntropyRegularization2022}. In particular, we use the loss function
\begin{equation}
	\mathcal{L}_\alpha(\btheta)=(1-\alpha)\mathcal{L}(\btheta)-\alpha H(p_\btheta|\knowledge{\by}),
\end{equation}
where we compute $H(p_\btheta|\knowledge{\by})=\frac{1}{3}(p_\btheta(\neg r, g) + p_\btheta(r, \neg g) + p_\btheta(\neg r, \neg g))$. 

We plot the results in Figure \ref{fig:entropy-independent} for the independent model and various values of $\alpha$. We see that the entropy regularisation does not help to calibrate the model. Rather, it biases the model more towards the $\neg r, \neg g$ vertex. For larger values of $\alpha$, the minimum of the augmented loss no longer is a minimum of the semantic loss. In fact, for $\alpha=0.1$, all initial points converge to a minimum that floats just above the bottom triangle. It assigns a probability of 0.023 to the impossible world $r, g$. 

The intuition for this is that the entropy regularisation is minimised only at the uniform belief $p_1$ from Figure \ref{fig:intro-fig}, which is not representable by independent distributions. Then, by minimising the augmented loss $\mathcal{L}_\alpha$, it trades off the closest point to $p_1$ and staying close to the bottom triangle. 

\begin{figure*}
	\centering
	\includegraphics[width=0.8\linewidth]{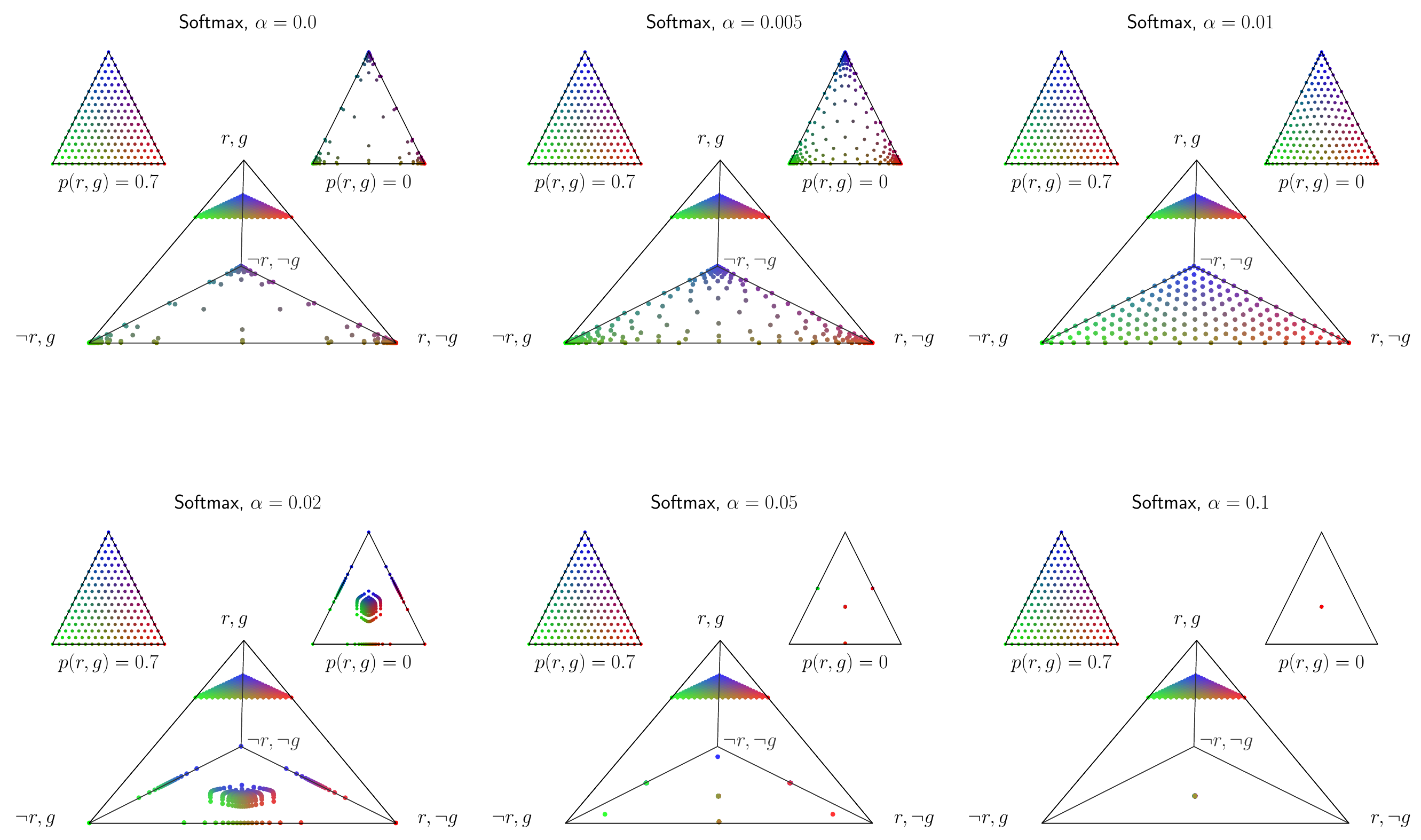}
	\caption{Minimising the semantic loss with entropy regularisation for the joint model.}
	\label{fig:entropy-joint}
\end{figure*}
For expressive distributions, the effect of entropy regularisation is quite different, which we plot in Figure \ref{fig:entropy-joint}. Again, the parameter $\alpha$ trades off the original minima, which are close to the vertices in the bottom triangle, to the uniform distribution $p_1$. For $\alpha=0.1$, the minima are indeed all close to $p_1$. However, for appropriate values of $\alpha$ such as $\alpha=0.01$, the minima almost distribute perfectly over the bottom triangle, and it almost finds the conditioned version of the original distribution. We note that finding the parameter $\alpha$ to get this behaviour would be extremely challenging in practice.

\section{The mixture of independent distributions}
\label{appendix:mixture}

In this appendix, we study the mixture of independent distributions with $k$ components. The main results here are that for $k\geq 2$, the space of possible distributions is connected. Furthermore, we provide bounds on the number of components needed to completely fill the space of all possible distributions $\poss{\by}$. A lower bound is the number of disconnected components in the prime implicant graph, and an upper bound is the number of prime implicants. This lower bound can be tricky: For example, the number of disconnected components in the MNIST Addition task is exponential in the number of digits.

The parameter space of this distribution is the Cartesian product $\Theta_k=\Delta^{k}\times[0, 1]^ {k\cdot n}$, where $\Delta^ k$ is the $k+1$-dimensional simplex. We map this parameter space to the space of distributions over worlds $\distributions$ with the map $f_{k\cdot \indep}:\Theta_k\rightarrow \distributions$, which we define as:
\begin{equation}
	\label{eq:indep-mixture}
	f_{k\cdot \indep}(\boldsymbol{\alpha}, \probs_1, ..., \probs_k)_i = \sum_{m=1}^k \alpha_m f_\indep(\probs_m)_i
\end{equation}
where $f_\indep$ is defined as in Equation \ref{eq:independentfunction}. Unlike $f_\indep$, this is not a bijection, as multiple parameterisations can map to the same distributions. We will refer to $p_\btheta=f_{k\cdot \indep}(\btheta)$ as a vector in $\distributions$.

\begin{lemma}
	\label{lemma:indep-mixture}
	Consider a parameter $\btheta \in \Theta_k$ of a mixture of $k$ independent components. Then $f_{k\cdot \indep}(\btheta)$ is a possible distribution if and only if all components $i$ such that $\alpha_i>0$ have a deterministic assignment that is an implicant.
\end{lemma}
\begin{proof}
	The mixture of independents assigns some mass to all independent distributions with $\alpha_i > 0$. For such an independent distribution to be possible, by Proposition \ref{thm:independence}, the deterministic assignment of this component has to be an implicant. Conversely, assume there is a component with $\alpha_i > 0$ that is not a possible distribution. Then, it assigns some mass to an impossible world, and so must the mixture of components. Therefore, the mixture is not a possible distribution.
\end{proof}

Let us now discuss the set of possible distributions for the mixture distribution $\possmix{k}=f_{k\cdot \indep}(\Theta_k) \cap \poss{\by}$.
Conveniently, if $k\geq 2$, the set of minima of the semantic loss under the mixture distribution is connected:
\begin{proposition}
	The set of possible mixture distributions $\possmix{k}$ is connected for $k\geq 2$.
\end{proposition}
\begin{proof}
	Since $\possindependent \subseteq \possmix{k}$, any two points that are connected in $\possindependent$ are also connected in $\possindependent_k$.

	Consider two points $\probs_1, \probs_2\in \possindependent$ that are not connected. Then we can create a convex combination between $p_{\probs_1}$ and $p_{\probs_2}$ in $\possmix{k}$ by moving $\alpha_1$ from 1 to 0 and $\alpha_2$ from 0 to 1. This convex combination is a possible distribution in $\possmix{k}$ by Lemma \ref{lemma:indep-mixture}. 
	
	Consider two parameters $\btheta_1, \btheta_2 \in \Theta_k$. We will construct a path from $p_{\theta_1}$ to $p_{\theta_2}$ through $\possmix{k}$. First, choose a component $i\in \{1, ..., k\}$ such that $\alpha_{1, i}>0$. Next, continuously map the convex mixture parameter from $\boldsymbol{\alpha}_1$ to $\mathbf{e}_i$ (the $i$-th standard normal vector). Call the resulting parameter $\hat{\btheta}_1$. Since we do not change the mixture components themselves, we never leave $\possmix{k}$ by Lemma \ref{lemma:indep-mixture}. Then, $p_{\hat{\btheta}_1}=f_\independent(\probs_{1, i})=p_{\probs_{1, i}}$ is a possible independent distribution. Consider some component $j\in \{1, ..., k\}$ such that $\alpha{2, j}>0$. As argued in the paragraphs above, there is a path between $p_{\probs_{1, i}}$ and $p_{\probs_{2, j}}$ in $\possmix{k}$. Consider $\hat{\btheta}_2$ to be $\btheta_2$ but replacing $\boldsymbol{\alpha}$ with $\mathbf{e}_j$ such that $p_{\hat{\btheta}_2}=p_{\probs_{2, j}}$. Then, we continuously map $\mathbf{e}_j$ to $\alpha_{2, j}$ to finally arrive at $\btheta_2$. 
\end{proof}

Clearly, increasing the number of components is beneficial to further covering the complete set of all possible distributions $\poss{\by}$. But how parameter-efficient is the use of mixtures of independent distributions to cover this set?  First, we prove a straightforward lemma.

\begin{lemma}
	\label{lemma:mixture-cover}
	Consider a parameter $\btheta \in \Theta_k$ of a mixture of $k$ independent components such that $p_\btheta\in \possmix{k}$. Then ${p_\btheta}_i> 0$ if and only if there is a component $m\in \{1, ..., k\}$ such that $\alpha_m > 0$ and the deterministic assignment of $\probs_m$ covers $\bw$.
\end{lemma}
\begin{proof}
	Let ${p_\btheta}_i> 0$. Then by Equation \ref{eq:indep-mixture} there must be a $m\in \{0, ..., k\}$ with $\alpha_m > 0$ such that $f_\indep(\probs_m)_i > 0$. But then, by Proposition \ref{thm:independence}, the deterministic assignment of $\probs_m$ covers $\bw$. 

	Similarly, let the deterministic assignment of $\probs_m$ cover $\bw$ and let $\alpha_m > 0$. Then by Proposition \ref{thm:independence}, $f_\indep(\probs_m)_i > 0$. But then by Equation \ref{eq:indep-mixture}, ${p_\btheta}_i> 0$.
\end{proof}

We next prove a significant lower bound:

\begin{theorem}
	\label{thm:lower-bound-mixture}
	The minimal number of mixture components needed to assign some probability to all possible worlds is the number of prime implicants in a \emph{minimal cover} of prime implicants $\mathcal{I}$. 
\end{theorem}
\begin{proof}
	First, we prove that if a mixture distribution can assign some probability to all possible worlds, then it has at least $|\mathcal{I}|$ components. Assume otherwise. Then $|\mathcal{I}|-1$ components are enough to cover $\poss{\by}$. Consider some distribution $p\in \poss{\by}$ such that $p_i>0$ for all possible worlds $\bw_i$. Let $\btheta\in \Theta_{|\mathcal{I}|-1}$ be parameters such that $p_\btheta=p$, which have to exist by assumption. 
	
	Let $\mathcal{I}'$ be the implicants formed from the independent parameters $\probs_i$, $i\in \{1, ..., I-1\}$. By Lemma \ref{lemma:mixture-cover}, the set of worlds such that $p_\btheta(\bw) > 0$ is $\mathcal{W}'=\bigcap_{\implicant\in \mathcal{I}'} \mathcal{W}_{\implicant}$. This set must equal $\mathcal{W}_{\knowledge{\by}}$ by the assumption that $p_i > 0$ for all possible worlds. But this is a contradiction, since this would make the set of implicants $\mathcal{I}'$ a cover of $\knowledge{\by}$ with $|\mathcal{I}|-1$ components, which is smaller than the minimal cover of prime implicants $\mathcal{I}$.

	Next, we prove that if the number of components is at least $|\mathcal{I}|$, then we can assign some probability to all possible worlds. Define an order $\implicant_1, ..., \implicant_{|\mathcal{I}|}$ of a minimal cover of prime implicants. Let $\probs_1, ..., \probs_{|\mathcal{I}|}$ be independent parameters such that $\probs_i$ is in the relative interior of $\Pimplicant_i$ (see Theorem \ref{prop:union} and its proof for a rigorous definition). Use $\probs_1, ..., \probs_{|\mathcal{I}|}$ together with $\boldsymbol{\alpha}$ such that $\alpha_i>0$ for all $i\in \{0, ..., |\mathcal{I}|\}$ to define parameters $\btheta\in \Theta_{|\mathcal{I}|}$. Then, by Lemma \ref{lemma:mixture-cover}, $p_\btheta$ assigns some probability to all possible worlds.
\end{proof}

Interestingly, here a minimal cover of prime implicants is relevant, while for Theorem \ref{prop:union}, we needed to consider the set of \emph{all} prime implicants (see also Appendix \ref{appendix:minimal-cover}). Figure \ref{fig:fourvars} provides some intuition: $b\wedge \neg c$ and $a\wedge c$ form a minimal set of prime implicants. By mixing between points on the line segments $\mathcal{P}_{b\wedge \neg c}$ and $\mathcal{P}_{a\wedge c}$, we can, in fact, cover the entire set of possible worlds.

Clearly, $|\mathcal{I}|$ is also a lower bound on the number of components needed to completely cover $\knowledge{\by}$. A simple upper bound for the number of components needed to mix is $|\mathcal{W}_{\knowledge{\by}}|$, since we can put a (deterministic) independent distribution on each of the possible worlds and mix them via $\boldsymbol{\alpha}$. Another lower bound is $\lceil |\mathcal{W}_{\knowledge{\by}}| / (\Wdim + 1)\rceil$, since $f_{k, \indep}(\Theta_k)$ is at most a $k\cdot(\Wdim + 1)$-dimensional subspace of $\distributions$.

Given Theorem \ref{thm:lower-bound-mixture} and the other bounds, is using a mixture of independents a parameter-efficient way to allow perception models to express more uncertainty? We would argue not, at least not in general. For example, the size of the minimal cover for MNIST Addition grows exponentially with the number of digits considered, in fact, it is equal to the number of possible worlds. But then we are using $|\mathcal{W}_{\knowledge{\by}}|\cdot(\Wdim + 1)$ parameters, which are $\Wdim$ times more parameters than necessary: The space of possible distributions is a $|\mathcal{W}_{\knowledge{\by}}|$-dimensional subspace of $\distributions$.

\section{Convexity of semantic loss}
\label{appendix:convexity}
In this Appendix, we show that the semantic loss is a convex loss over the space of all possible distributions $\distributions$ using Jensen's inequality and the fact that the WMC in Equation \ref{eq:wmc} is linear. Note that this does not mean it is convex with respect to the parameters $\btheta$ of the perception model. Let $p_1, p_2\in \distributions$. Note that since $\distributions$ is a convex set, $\lambda p_1 + (1-\lambda) p_2 \in \distributions$. Then,
\begin{align*}
&\mathcal{L}(\lambda p_1 + (1-\lambda)p_2) \\
=& -\log \Big(\sum_{\bw\in\mathcal{W}_{\knowledge{\by}}}\lambda p_1(\bw) 
+ (1-\lambda) p_2(\bw)\Big) \\	
=& -\log \Big(\lambda \sum_{\bw\in\mathcal{W}_{\knowledge{\by}}} p_1(\bw) + (1-\lambda) \sum_{\bw\in\mathcal{W}_{\knowledge{\by}}} p_2(\bw)\Big) \\
\leq& -\lambda \log \sum_{\bw\in\mathcal{W}_{\knowledge{\by}}} p_1(\bw) - (1-\lambda) \log \sum_{\bw\in\mathcal{W}_{\knowledge{\by}}} p_2(\bw) \\
=& \lambda \mathcal{L}(p_1) + (1-\lambda) \mathcal{L}(p_2)
\end{align*}

\section{Cubical sets generated by prime implicants}
\label{appendix:cubical}

To help understand our results geometrically and prove some of the main theorems in Appendix \ref{appendix:proofs}, we study the basic properties of the cubical set $\cubepossible$. For background on polytopes, faces, face posets, and polyhedral complexes, see \citet{zieglerLecturesPolytopes1995}, and for an introduction and basic properties of cubical sets, see \citet{kaczynskiComputationalHomology2004}.

First, we define elementary cells, which allow us to access the relative interior of a cube by changing intervals from a closed set to an open set.
\begin{definition}
	Associated with each cube $C=I_1 \times ... \times I_n$ is an \emph{(elementary) cell} $\cell=\cellint_1 \times ... \times \cellint_n \subseteq C$, where each $\cellint_i=I_i$ for the degenerate intervals $[0, 0]$ and $[1, 1]$, and $\cellint_i = (0, 1)$ for the nondegenerate interval $[0, 1]$. 
\end{definition}

The following proposition allows us to associate implicants to faces of $\cubepossible$.
\begin{proposition}
	\label{prop:faces_are_implicants}
	The faces $\mathcal{C}(\cubepossible)$ of $\cubepossible$ is the set of implicant cubes.
\end{proposition}
\begin{proof}
	Consider some face $X\in \mathcal{C}(\cubepossible)$. Since by definition a face is an (elementary) cube, it can be represented by $I_1 \times ... \times I_n$, where each $I_i$ is an elementary interval. Use the degenerate intervals to create a partial assignment $\implicant$. If $\implicant$ was not an implicant, then by Theorem \ref{thm:independence},  any $\probs\in \cellimplicant$ is not possible, which contradicts Theorem \ref{prop:union}. Therefore, $X$ is an implicant cube.

	Next, consider some implicant $\implicant$. By definition, there is a prime implicant $\altimplicant\subseteq \implicant$ that assigns to a subset of $\implicant$. Therefore, the only difference between the implicant cubes of $\implicant$ and $\altimplicant$ is that the latter has fewer nondegenerate intervals. Therefore, $C_\implicant \subseteq C_{\altimplicant}$, and so $C_\implicant$ is a face of $C_{\altimplicant}$. Therefore, $C_\implicant$ is a face of $\cubepossible$.
\end{proof}

\begin{proposition}
	The facets of $\cubepossible$ are the prime implicant cubes $\Pimplicant$.
\end{proposition}
\begin{proof}
	Consider some facet $X$ of $\cubepossible$. By Proposition \ref{prop:faces_are_implicants}, the deterministic part of $X$ is an implicant $\implicant$. Assume $\implicant$ is not a prime implicant. Then there is a deterministic variable $i$ that we can remove from $\implicant$ and still have an implicant $\altimplicant$. But then $\Paltimplicant\supset C$, with $C=$ being a face of $\Paltimplicant$, which is in contradiction with the assumption that $C$ is a facet.
\end{proof}

\begin{proposition}
	\label{prop:vertices_are_worlds}
	The vertices $\mathcal{C}_0(\cubepossible)$ of $\cubepossible$ is equal to the set of possible worlds $\mathcal{W}_{\knowledge{\by}}$.
\end{proposition}
\begin{proof}
	Let $\mathcal{C}_0(\cubepossible)\subseteq \{0, 1\}^n$ be the vertices of $\cubepossible$, which by Proposition \ref{prop:faces_are_implicants} is the implicant cubes with no stochastic variables, that is, it assigns a value to every variable and corresponds directly to a world. By the fact that it is an implicant, this world has to be possible, that is, $\mathcal{C}_0(\cubepossible) = |\mathcal{W}_{\knowledge{\by}}|$.
\end{proof}

\begin{proposition}
	\label{prop:vertices_are_covers}
	The vertices $\mathcal{C}_0(\Pimplicant)$ of (prime) implicant cube $\Pimplicant$ is equal to the cover of $\implicant$.
\end{proposition}
\begin{proof}
	Considering $\implicant$ as the constraint that a world $\bw$ has to agree on the deterministic variables with $\implicant$, by Proposition \ref{prop:vertices_are_worlds}, the vertices of $\Pimplicant$ are precisely such worlds. This is the cover of $\implicant$.
\end{proof}

\section{Proofs of the main theorems}
\label{appendix:proofs}
In this appendix, we give the proofs for the theorems in the main paper. Understanding some of these proofs requires understanding the connection of our problem to cubical sets, which we give in Appendix \ref{appendix:cubical}. We recommend going through Appendix \ref{appendix:cubical} before reading the proofs.

We start off by defining the transformation from independent parameters to distributions and prove that it is a bijection. First, let 
\begin{align}
	\label{eq:independentfunction}
	f_\indep(\probs)_i &= \prod_{j=1}^\Wdim \varprob_j^{w_{i, j}} \cdot (1-\varprob_j)^{1 - w_{i, j}}.
\end{align}
be a function $f_\indep: [0, 1]^n\rightarrow \independent$ that maps the parameters $\probs$ to the set of independent distributions. Note that this is the transformation used in Equation~\ref{eq:independentset}.
\begin{lemma}
	\label{lemma:bijection}
	The map $f_\indep$ is a continuous bijection from $[0, 1]^n$ to $\independent$\footnote{It is a bijection to $\independent$, but not to the codomain $\distributions$.}.
\end{lemma}
\begin{proof}
	Define the function $f^{-1}_\indep: \independent\rightarrow [0, 1]^n$ as
	\begin{equation}
		f^{-1}_\indep(p)_i= p(w_i=1)= \sum_{j=1}^{|\mathcal{W}|}w_{j, i} p_j \quad i\in 1, ..., \Wdim.
	\end{equation}
	Consider $\probs \in [0, 1]^n$. Since $\probs$ are the parameters of an independent distribution, the marginal probability $p_\probs(w_i=1)=\mu_{i}$. This is also by definition the sum of the probabilities of all worlds $\bw_k$ with $w_{k, i}=1$, that is, $f^ {-1}_\indep$. Therefore, $f^{-1}_\indep(f_\indep(\probs))_i=\mu_i$. 

	Next, consider $p\in \independent$. By the definition of $\independent$ in Equation $\ref{eq:independentset}$, there must be a parameter $\probs\in [0, 1]^n$ such that $f(\probs)=p$, that is, $p$ represents an independent distribution by definition. Therefore, the marginal probabilities computed with $f^{-1}_\indep$ precisely describe $p$, and so $f_\indep(f^{-1}_\indep(p))=p$. 
\end{proof}

Next, we repeat the theorems from the main body of the text and give their proofs. 
\end{document}